\documentclass[lettersize,journal]{IEEEtran}
\usepackage[T1]{fontenc}
\usepackage{amsmath,amsfonts}
\usepackage{algorithmic}
\usepackage{algorithm}
\usepackage{array}
\usepackage[caption=false,font=normalsize,labelfont=sf,textfont=sf]{subfig}
\usepackage{textcomp}
\usepackage{stfloats}
\usepackage{url}
\usepackage{verbatim}
\usepackage{graphicx}
\usepackage{cite}
\usepackage{booktabs} 
\usepackage{multirow}
\usepackage{balance}
\usepackage{bbm}
\usepackage{amsthm}
\usepackage{hyperref}
\newtheorem{theorem}{Theorem}[section]

\newtheorem{proposition}[theorem]{Proposition}
\newtheorem{corollary}[theorem]{Corollary}
\newtheorem{definition}{Definition}[section]
\newtheorem{assumption}{Assumption}[section]
\newtheorem{remark}{Remark}[section]

\begin{document}

\title{Depth as Prior Knowledge for Object Detection}

\author{Moussa Kassem Sbeyti\quad Nadja Klein\\
Scientific Computing Center, Karlsruhe Institute of Technology\\
\tt\small{\{moussa.sbeyti, nadja.klein\}@kit.edu}\\
        % <-this % stops a space
% \thanks{Manuscript received February 05, 2026}
}

% \markboth{IEEE Transactions on Pattern Analysis and Machine Intelligence}%

% \IEEEpubid{\copyright~2026 IEEE} %0000--0000/00\$00.00~

\maketitle

\begin{abstract}
Detecting small and distant objects remains challenging for object detectors due to scale variation, low resolution, and background clutter. Safety-critical applications require reliable detection of these objects for safe planning. Depth information can improve detection, but existing approaches require complex, model-specific architectural modifications. We provide a theoretical analysis followed by an empirical investigation of the depth-detection relationship. Together, they explain how depth causes systematic performance degradation and why depth-informed supervision mitigates it. We introduce DepthPrior, a framework that uses depth as prior knowledge rather than as a fused feature, providing comparable benefits without modifying detector architectures. DepthPrior consists of Depth-Based Loss Weighting (DLW) and Depth-Based Loss Stratification (DLS) during training, and Depth-Aware Confidence Thresholding (DCT) during inference. The only overhead is the initial cost of depth estimation. 
Experiments across four benchmarks (KITTI, MS COCO, VisDrone, SUN RGB-D) and two detectors (YOLOv11, EfficientDet) demonstrate the effectiveness of DepthPrior, achieving up to +9\% mAP$_S$ and +7\% mAR$_S$ for small objects, with inference recovery rates as high as 95:1 (true vs.~false detections). DepthPrior offers these benefits without additional sensors, architectural changes, or performance costs. Code is available at \url{https://github.com/mos-ks/DepthPrior}.
\end{abstract}
\begin{IEEEkeywords}
Adaptive thresholding, geometric constraints, guided supervision, monocular depth
\end{IEEEkeywords}

%%%%%%%%% INTRODUCTION
\section{Introduction}
Object detectors achieve strong performance on standard benchmarks, though their performance varies across object distances. In autonomous driving, a pedestrian close to the camera occupies a large image area with rich visual features. The same pedestrian at a greater distance occupies fewer pixels with minimal distinguishable characteristics. This geometric transformation affects all detectors, yet standard supervised training treats objects uniformly regardless of distance. Despite this, the systematic relationship between object distance and detection performance has received limited attention.

Current monocular depth estimators use self-supervised learning to accurately predict depth (and thus distance) at scale~\cite{yang2024depth}, while state-of-the-art detectors still rely on supervised learning~\cite{ciocarlan2024self}. Depth information can improve supervision signals, but existing approaches require invasive architectural modifications that limit practical deployment, e.g., multi-branch networks~\cite{zhang2023monodetr}, specialized convolution operators~\cite{ding2020learning}, or complex feature fusion~\cite{gupta2014learning,song2017depth}. Moreover, depth-informed inference remains unstudied. Distant objects receive systematically lower confidence scores even when correctly detected. However, standard post-processing applies uniform thresholds across the depth range, discarding valid detections.

Rather than treating depth as features to be fused, we use it as prior knowledge to inform training and inference. This requires no architectural changes, maintaining independence from detector choice or deployment domain. We introduce \emph{DepthPrior}, a framework with three components. Depth-Based Loss Weighting (DLW) weights each object's loss based on depth, scaling contributions across the batch (inter-image depth relationship). Depth-Based Loss Stratification (DLS) decomposes detection loss per image into close and distant intervals with separate weights (intra-image depth relationship). Depth-Aware Confidence Thresholding (DCT) learns depth-dependent confidence thresholds via splines, optimizing true vs.~false detection trade-offs during inference.

Our main contributions are:
\begin{itemize}
\item We theoretically formalize the depth-detection relationship. As distance increases, confidence decreases and missing detections increase, revealing a mismatch between uniform supervision and the geometric nature of detection difficulty.
\item We introduce a modular framework that exploits depth to inform training and inference without architectural modifications, maintaining deployment flexibility.
\item We empirically investigate the depth-dependency of object detectors and validate DepthPrior across four benchmarks (indoor/outdoor, aerial/ground-level) and two detector families (anchor-based and anchor-free).
\end{itemize}
%%%%%%%%% RELATED WORK
\section{Related Work} \label{sec:rel_work}
\subsection{Depth-Aware Object Detection}
Monocular depth estimation has progressed from CNN-based~\cite{eigen2014depth} to transformer-based~\cite{ranftl2021vision} models. 
Recent approaches such as DepthAnything~\cite{yang2024depth} achieve strong zero-shot generalization. This enables reliable depth integration into detection pipelines without additional sensors.

Depth-aware object detection has evolved through distinct paradigms balancing the trade-off between performance and architectural complexity. Early works~\cite{gupta2014learning,hou2016deeply} extract complementary features (e.g., geometric contours) from depth maps and fuse them with RGB features. Other approaches adopt weak supervision followed by joint fine-tuning~\cite{song2017depth}. For 3D detection, Ding et al.~\cite{ding2020learning} introduce depth-guided convolutions where kernels adapt to local depth from LiDAR. Zhang et al.~\cite{zhang2023monodetr} guide the decoder of transformers with parallel depth and visual encoders connected via cross-attention. Cetinkaya et al.~\cite{cetinkaya2022does} show that using raw depth without feature extraction decreases detection performance. Depth must first be transformed into meaningful features, either manually~\cite{hou2016deeply} or through learned encoders~\cite{song2017depth}. However, these methods require careful architecture design, may suffer from modality imbalance, and incur high computational overhead. In contrast, DepthPrior treats depth as prior knowledge guiding training and inference, avoiding feature fusion overhead and enabling plug-and-play integration with any detector.
 
\subsection{Multi-Task Learning for Object Detection} 
An alternative strategy to depth-aware detection is multi-task learning (MTL). MTL jointly optimizes depth estimation and detection through shared representations~\cite{kendall2018multi,chen2018multi,yu2021yolo}. To mitigate negative transfer, uncertainty-weighted balancing~\cite{kendall2018multi} and learnable weights~\cite{liebel2018auxiliary} have been introduced. However, negative transfer remains problematic as task affinities depend on dataset characteristics~\cite{standley2020tasks,vandenhende2021multi}. In MTL, the additional decoders required for each task increase computational cost, offsetting the benefits of parameter sharing in the encoder.

In contrast, DepthPrior uses depth post estimation, integrating into existing detectors without retraining depth-specific backbones or task-specific heads.

\subsection{Prior Knowledge in Object Detection} 
A third category of depth-aware detection leverages geometric constraints~\cite{zhang2021objects}. Chen et al.~\cite{chen2020monopair} predict 3D distances between adjacent objects. They use pairwise spatial relationships as constraints in a joint optimization process to refine locations of occluded objects. Pseudo-LiDAR~\cite{wang2019pseudo} generates pseudo point clouds from depth for monocular 3D detection. Beyond depth, other priors include surface normals~\cite{bansal2016marr}, segmentation-guided attention~\cite{miao2022prior}, and occurrence frequency of objects and their color as scene weights for classification~\cite{ding2018prior}. 

DepthPrior uses depth not to extend 2D detection to 3D, but to improve performance on distant and small objects.

\subsection{Confidence Calibration and Threshold Optimization}
Depth in training has received more attention than in inference. Post-hoc calibration methods such temperature scaling~\cite{guo2017calibration} address miscalibration uniformly without considering geometric structure. Non-maximum suppression~\cite{hosang2017learning} operates on spatial overlap without accounting for depth-dependent characteristics. Test-time adaptation~\cite{wang2020tent} demonstrates that inference-time adjustments can improve performance, but none consider depth-aware post-processing. DepthPrior introduces geometric structure into threshold optimization, learning depth-dependent adjustments reflecting varying difficulty across distances.
%%%%%%%%% METHODS
\section{Theoretical Considerations} \label{sec:dd_rel_methods}
Distant objects appear smaller and are harder to detect, while nearby objects are large and easy to detect. Standard detectors use depth-uniform loss functions, implicitly assuming detection difficulty is independent of depth. We argue this assumption is flawed. Consider autonomous driving: a car at $10$~meters occupies $200\times150$ pixels with rich features and consistent detection losses. The same car at $80$~meters reduces to $25\times18$ pixels with high loss variance. This variance arises from multiple compounding factors. Sources of noise that are negligible for nearby objects, such as pixel noise, background interference, and annotation errors, become significant perturbations for distant ones. We formalize how loss variance scales with depth, contradicting the homoscedastic assumption in standard supervised training.

This section builds the methodological fundament for the components of DepthPrior presented in the next Section \ref{sec:methods}. 

\subsection{Variance Model of Detection Loss}
\begin{definition}[Signal Quality Function]
\label{def:signal_quality}
For an object at distance $d \in \mathbb{R}^+$ from the camera, we define the signal quality function $Q: \mathbb{R}^+ \to \mathbb{R}^+$ as a measure of the visual information available in a captured image to an object detector.
\end{definition}

We use this definition to model how detection loss variance increases as visual information degrades. We make the following assumptions about the relationship between distance, signal quality, and loss variance.

\begin{assumption}[Inverse-Square Law] \label{assumption:intensity_distance}
The signal quality function follows the inverse-square law governing light intensity and distance~\cite{hecht1974optics}:
\begin{equation}
Q(d) = \frac{\kappa}{d^2},
\end{equation}
where $\kappa \in \mathbb{R}^+$ aggregates camera intrinsic parameters and object-specific properties. This is empirically validated in computer vision~\cite{hao2023understanding}. 
\end{assumption}

\begin{assumption}[Variance-Signal Relationship] \label{assumption:variance_signal}
Detection loss variance increases as signal quality degrades. Aleatoric uncertainty (noise inherent in observations) varies with input characteristics such as resolution and occlusion~\cite{kendall2017uncertainties}. Since distant objects exhibit degraded input quality, we model their loss variance as inversely related to signal quality:
\begin{equation}
\text{Var}[\mathcal{L} \mid d] = \frac{\alpha^2}{Q(d)} + \sigma_\varepsilon^2,
\end{equation}
where $\alpha^2 > 0$ is the signal-dependent variance scaling parameter and $\sigma_\varepsilon^2 > 0$ represents the irreducible, depth-independent variance arising from sources such as annotation noise or sensor noise.
\end{assumption}

\begin{proposition}[Depth-Induced Heteroscedasticity]
\label{prop:depth_heteroscedasticity}
Under Assumptions~\ref{assumption:intensity_distance} and~\ref{assumption:variance_signal}, the conditional variance of the detection loss is:
\begin{equation}
\text{Var}[\mathcal{L} \mid d] = \sigma_0^2 d^2 + \sigma_\varepsilon^2,
\end{equation}
where $\sigma_0^2 = \alpha^2/\kappa$.
\end{proposition}
This contradicts the homoscedastic assumption implicit in standard supervised training. See Supplementary Material Section~A for proof.

\subsection{From Visual Information Degradation to Training Bias}
This heteroscedasticity (Proposition~\ref{prop:depth_heteroscedasticity}) induces systematic training bias.

\begin{corollary}[Bias Toward Nearby Objects]
\label{cor:training_bias}
Under stochastic gradient descent (SGD) with uniform weighting, objects with lower loss variance (nearby) are fitted before those with higher loss variance (distant). Given a fixed training budget, this results in systematic underfitting of distant objects.
\end{corollary}
The proof combines Proposition~\ref{prop:depth_heteroscedasticity} with empirical observations that neural networks prioritize low-variance examples~\cite{arpit2017closer,toneva2019empirical}; see Supplementary Material Section~A.

\begin{remark}
Samples that yield consistent, low-variance losses are 
fitted early and remain ``unforgettable''. High-variance samples are fitted later and forgotten more frequently~\cite{toneva2019empirical}. This motivates curriculum learning strategies that explicitly present easy samples first~\cite{bengio2009curriculum,kumar2010self}. Our contribution identifies depth as a systematic source of loss variance heteroscedasticity in object detection. Nearby objects exhibit low loss variance (analogous to ``unforgettable'' samples), while distant objects exhibit high loss variance (analogous to ``forgettable'' samples). This disadvantages distant objects.
\end{remark}

\subsection{Compensating for Training Bias}
\label{sec:methodswL}
We introduce two independent mechanisms to counteract the depth-dependent training bias established in corollary~\ref{cor:training_bias}.

\begin{definition}[Variance-Compensating Weights]
\label{def:variance_weights}
We define variance-compensating weights based on Proposition~\ref{prop:depth_heteroscedasticity} and Corollary~\ref{cor:training_bias} as
\begin{equation}
w^*(d) \propto \text{Var}[\mathcal{L} \mid d] = \sigma_0^2 d^2 + \sigma_\varepsilon^2.
\end{equation}
For large depths where $d \gg \sigma_\varepsilon/\sigma_0$, this simplifies to $w^*(d) \propto d^2$.
\end{definition}

\begin{remark}
While the aleatoric uncertainty inherent to distant objects cannot be eliminated as the visual information is degraded, variance-compensating weights prevent this uncertainty from biasing optimization toward easier nearby examples. This relates to importance sampling approaches that prioritize informative samples~\cite{katharopoulos2018not} and hard example mining methods that emphasize difficult detections~\cite{shrivastava2016training}.
\end{remark}

Beyond continuous reweighting, stratification enables interpretable control over learning dynamics.

\begin{definition}[Depth-Based Loss Stratification]
\label{def:stratified_loss}
Given a partition of the depth range into $K$ disjoint intervals $\{S_1, \ldots, S_K\}$ where $S_k = [d_{k-1}, d_k)$, the stratified loss decomposes the total training loss into interval-specific components:
\begin{equation}
\mathcal{L}_{\text{strat}} = \sum_{k=1}^K \lambda_k \hat{\mu}_k, \quad \hat{\mu}_k = \frac{1}{|S_k|} \sum_{i : d_i \in S_k} \mathcal{L}_i,
\end{equation}
where $\lambda_k > 0$ are the interval weights, and $|S_k|$ denotes the number of samples in interval $S_k$.
\end{definition}

\begin{proposition}[Gradient of $\mathcal{L}_{\text{strat}}$]
\label{prop:stratification_gradient}
Let $\theta$ denote the model parameters and $\eta$ the learning rate. The gradient of the stratified loss decomposes as:
\begin{equation}
\nabla_\theta \mathcal{L}_{\text{strat}} = \sum_{k=1}^K \frac{\lambda_k}{|S_k|} \sum_{i : d_i \in S_k} \nabla_\theta \mathcal{L}_i.
\end{equation}
Consequently, samples in interval $S_k$ effectively experience learning rate $\eta \lambda_k / |S_k|$ in the SGD update $\theta_{t+1} = \theta_t - \eta \nabla_\theta \mathcal{L}_{\text{strat}}$.
\end{proposition}

\begin{remark}
Stratification allows practitioners to specify high-level objectives such as ``prioritize detection beyond 50 meters'' by simply adjusting $\lambda_k$. This is particularly valuable in application-specific contexts such as long-range obstacle detection in autonomous driving.

Since Proposition~\ref{prop:stratification_gradient} requires only coarse depth binning rather than precise per-pixel depth values, it is robust to the inevitable errors in monocular depth estimation.
\end{remark}

\begin{remark}
\label{remark:mtl_connection}
Depth-stratified loss can be viewed as a structured MTL problem where each interval $S_k$ defines a sub-task with shared parameters but independent loss weighting.
\end{remark}

\subsection{From Training to Inference}
\label{sec:dct_theory}
Detectors predict confidence scores $s\in\mathbb{R}^+$ and filter predictions below threshold $\tau_0$. Uniform thresholds fail to account for systematic score variations across the depth range.

\begin{assumption}[Thresholding Costs]
\label{assumption:costs}
Let $C_{\text{FN}} > 0$ and $C_{\text{FP}} > 0$ denote the costs for rejecting true detections and accepting false ones, respectively. These costs reflect application-specific priorities and are assumed fixed~\cite{sbeyticost}.
\end{assumption}

\begin{theorem}[Depth-Dependent Optimal Threshold]
\label{thm:optimal_threshold}
Let $P_{\text{TP}}(s, d)$ and $P_{\text{FP}}(s, d) = 1 - P_{\text{TP}}(s, d)$ denote the probabilities that a detection with score $s$ at depth $d$ is a true or false positive, respectively. Under Assumption~\ref{assumption:costs}, the cost-minimizing threshold is:
$\tau^*(d) = \inf \left\{ s : \frac{P_{\text{TP}}(s, d)}{P_{\text{FP}}(s, d)} \geq \frac{C_{\text{FP}}}{C_{\text{FN}}} \right\}$.
\end{theorem}
See Supplementary Material Section~A for proof. 

\begin{remark}
The functions $P_{\text{TP}}(s, d)$ and $P_{\text{FP}}(s, d)$ can be estimated empirically from validation data by computing the fraction of detections at each $(s, d)$ that match ground-truth (GT) objects. If these functions vary with depth, as occurs when detectors exhibit systematic confidence miscalibration for distant objects, then the optimal threshold $\tau^*(d)$ will be depth-dependent. However, the precise nature of this variation depends on detector architecture, training methodology, and dataset characteristics. 
\end{remark}

\section{Methods}
\label{sec:methods}
Figure~\ref{fig:depthprior_framework} illustrates the components of DepthPrior: DLW and DLS based on Section~\ref{sec:methodswL}, and DCT based on Section~\ref{sec:dct_theory}.
\begin{figure*}
    \centering
    \includegraphics[width=1\linewidth]{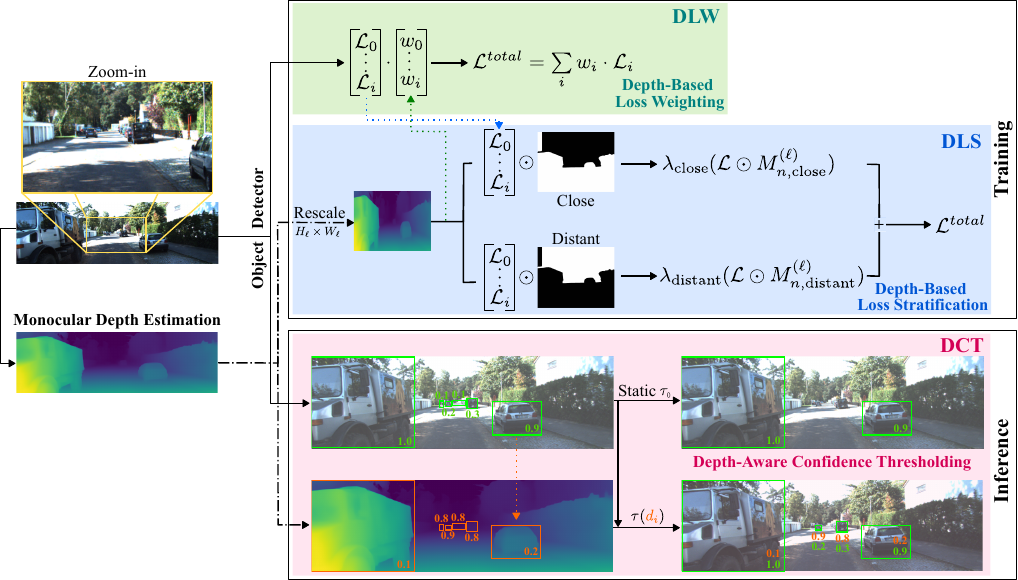}
        \caption{DepthPrior framework (notation simplified for clarity). Top (DLW): non-linear weighting $w_i = 1 + \alpha \cdot \exp(d_{i,\text{norm}})$ prioritizes distant objects. Middle (DLS): binary masks decompose loss into close/distant components with weights $\lambda_{\text{close}}, \lambda_{\text{distant}}$. Bottom (DCT): learned splines $\tau(d_{i,\text{norm}})$ adjust thresholds during inference. The framework requires only monocular depth estimation and operates during both training and inference without architectural modifications.}
    
    \label{fig:depthprior_framework}
\end{figure*}

\subsection{Depth-Based Loss Weighting (DLW)}
\label{subsec:dwl}
DLW implements variance-compensating reweighting (Definition~\ref{def:variance_weights}) in three steps.

\subsubsection{Normalization and Inversion}
Given depth maps $\{D_n\}_{n=1}^{N}$ for a batch of $N$ images, each $D_n \in \mathbb{R}^{H \times W}$ represents the predicted depth values for image $n$, with $H$ and $W$ denoting the image height and width, respectively. We normalize across all images:
\begin{equation}
D_{n,\text{norm}} = \frac{D_n - D_{\min}^{\text{batch}}}{D_{\max}^{\text{batch}} - D_{\min}^{\text{batch}}},
\end{equation}
where $D_{\min}^{\text{batch}} = \min_{n,h,w} D_n(h,w)$ and $D_{\max}^{\text{batch}} = \max_{n,h,w} D_n(h,w)$ denote the global minimum and maximum depth values across all spatial locations $(h,w) \in \{1,\ldots,H\} \times \{1,\ldots,W\}$ and all images $n \in \{1, \ldots, N\}$ in the batch. This yields normalized depth maps $D_{n,\text{norm}} \in [0, 1]^{H \times W}$.

Since monocular depth estimators output inverse depth (higher values indicate brighter pixels, i.e., closer objects), we invert to get distance proportional depth: $d_{i,\text{norm}} = 1 - D_{n,\text{norm}}(h_i, w_i)$, where $d_{i,\text{norm}} \in [0, 1]$ is the normalized depth for object $i$ with 1 being most distant.

\subsubsection{Exponential Weighting}
Motivated by Definition~\ref{def:variance_weights}, the depth-dependent weight for sample $i$ is computed using an exponential function:
\begin{equation}
w_i = 1 + \alpha \cdot \exp(d_{i,\text{norm}}),
\end{equation}
where $\alpha > 0$ controls far-object emphasis. The additive 1 ensures numerical stability (otherwise vanishing weights for small $\alpha$) and makes $\alpha$ a meaningful parameter (otherwise $\frac{\alpha e}{\alpha} = e \approx 2.72$). The weight ratio between farthest and closest objects is $(1 + \alpha e)/(1 + \alpha)$, varying from 1.15 ($\alpha=0.1$) to 2.56 ($\alpha=10$).

\subsubsection{Weighted Loss Computation}
For object $i$, let $\mathcal{L}_i^{\text{cls}}$ and $\mathcal{L}_i^{\text{box}}$ denote the classification and box regression losses, respectively. The weighted loss for a single training batch is:
\begin{equation}
\mathcal{L}^{\text{total}}_{\text{DLW}} = \frac{1}{N_b} \sum_{i=1}^{N_b} w_i \cdot \left(\mathcal{L}_i^{\text{cls}} + \mathcal{L}_i^{\text{box}}\right),
\end{equation}
where $N_b$ denotes the total number of objects across all $N$ images in the batch.

\subsection{Depth-Based Loss Stratification (DLS)}
\label{subsec:dsl}
DLS implements stratification (Proposition~\ref{prop:stratification_gradient}), enabling explicit control over the close/distant learning trade-off.

\subsubsection{Per-Image Depth Normalization}
Unlike DLW, DLS normalizes depth independently for each image. For image $n$ in the batch:
\begin{equation}
D_{n,\text{norm}} = 1 - \frac{D_n - \min(D_n)}{\max(D_n) - \min(D_n)},
\end{equation}
where $\min(D_n)$, $\max(D_n)$ are the minimum and maximum depth values within that image. The resulting $D_{n,\text{norm}} \in [0, 1]^{H \times W}$ assigns higher values to distant objects and adapts to each image's depth range.

\subsubsection{Binary Depth Stratification}
Using a threshold parameter $\beta \in (0, 1)$, the binary masks for image $n$ are:
\begin{align}
M_{n,\text{distant}}(h,w) &= \mathbbm{1}\left[D_{n,\text{norm}}(h,w) \geq \beta\right], \\
M_{n,\text{close}}(h,w) &= \mathbbm{1}\left[D_{n,\text{norm}}(h,w) < \beta\right],
\end{align}
where $\mathbbm{1}[\cdot]$ denotes the indicator function applied elementwise, and $(h,w) \in \{1,\ldots,H\} \times \{1,\ldots,W\}$ indexes spatial locations. The resulting masks $M_{n,\text{close}}, M_{n,\text{distant}} \in \{0,1\}^{H \times W}$ partition the image into close and distant regions, respectively. The threshold $\beta$ is a hyperparameter controlling the depth cutoff. While Definition~\ref{def:stratified_loss} permits arbitrary $K$-way partitions, we implement binary stratification $(K=2)$ with intervals $S_1 = [0, \beta)$ (close) and $S_2 = [\beta, 1]$ (distant) for interpretability. Extension to finer-grained depth intervals is straightforward by introducing additional thresholds $\beta_1 < \beta_2 < \cdots < \beta_{K-1}$.

\subsubsection{Stratified Loss Computation}
For detection architectures with feature pyramid networks, the loss is computed at multiple feature levels. Let $L$ denote the number of feature levels, with each level $\ell \in \{1, \ldots, L\}$ having spatial dimensions $H_\ell \times W_\ell$. At each level, the network produces classification logits $\hat{y}_\ell \in \mathbb{R}^{H_\ell \times W_\ell \times C}$ and bounding box predictions $\hat{b}_\ell \in \mathbb{R}^{H_\ell \times W_\ell \times 4}$, where $C$ is the number of classes. The corresponding GT labels are $y_\ell \in \{0, 1\}^{H_\ell \times W_\ell \times C}$ and boxes $b_\ell \in \mathbb{R}^{H_\ell \times W_\ell \times 4}$.

To compute the depth-stratified loss, the depth maps are first rescaled to match the spatial dimensions of each feature level. Then, the binary masks $M_{n,\text{close}}$ and $M_{n,\text{distant}}$ are extracted, yielding $M_{n,\text{close}}^{(\ell)}, M_{n,\text{distant}}^{(\ell)} \in \{0,1\}^{H_\ell \times W_\ell}$. The stratified losses for a batch of $N$ images are computed as:
\begin{equation}
\begin{aligned}
\mathcal{L}_{\text{close}}^{\text{cls}} &= \frac{1}{L} \sum_{\ell=1}^{L} \frac{1}{N} \sum_{n=1}^{N}  \sum_{h,w} \mathcal{L}^{\text{cls}}(\hat{y}_{n,\ell}(h,w), y_{n,\ell}(h,w))  \\
&\quad \cdot M_{n,\text{close}}^{(\ell)}(h,w), \\
\mathcal{L}_{\text{distant}}^{\text{cls}} &= \frac{1}{L} \sum_{\ell=1}^{L} \frac{1}{N} \sum_{n=1}^{N}  \sum_{h,w} \mathcal{L}^{\text{cls}}(\hat{y}_{n,\ell}(h,w), y_{n,\ell}(h,w))  \\
&\quad \cdot M_{n,\text{distant}}^{(\ell)}(h,w), \\
\mathcal{L}_{\text{close}}^{\text{box}} &= \frac{1}{L} \sum_{\ell=1}^{L} \frac{1}{N} \sum_{n=1}^{N}  \sum_{h,w} \mathcal{L}^{\text{box}}(\hat{b}_{n,\ell}(h,w), b_{n,\ell}(h,w)) \\
&\quad \cdot M_{n,\text{close}}^{(\ell)}(h,w), \\
\mathcal{L}_{\text{distant}}^{\text{box}} &= \frac{1}{L} \sum_{\ell=1}^{L} \frac{1}{N} \sum_{n=1}^{N}  \sum_{h,w} \mathcal{L}^{\text{box}}(\hat{b}_{n,\ell}(h,w), b_{n,\ell}(h,w))  \\
&\quad \cdot M_{n,\text{distant}}^{(\ell)}(h,w),
\end{aligned}
\end{equation}

\subsubsection{Weighted Stratified Loss}
The total loss with DLS combines the interval-specific losses:
\begin{equation}
\mathcal{L}^{\text{total}}_{\text{DLS}} =
\lambda_{\text{close}}\bigl(\mathcal{L}_{\text{close}}^{\text{cls}} + \mathcal{L}_{\text{close}}^{\text{box}}\bigr)
+ \lambda_{\text{distant}}\bigl(\mathcal{L}_{\text{distant}}^{\text{cls}} + \mathcal{L}_{\text{distant}}^{\text{box}}\bigr),
\end{equation}
where $\lambda_{\text{close}}, \lambda_{\text{distant}} > 0$ are interval-specific weights.

\subsection{Depth-Aware Confidence Thresholding (DCT)}
\label{subsec:dct}
DCT learns depth-dependent thresholds (Theorem~\ref{thm:optimal_threshold}) without retraining.

\subsubsection{Spline-Based Threshold Parameterization}
\label{subsec:spline_threshold}
Given a reference threshold $\tau_0 \in [0, 1]$, the depth-dependent threshold is:
\begin{equation}
\tau(d; \psi) = \tau_0 - g(d; \psi),
\end{equation}
where $g(d; \psi)$ is a smooth adjustment function constructed via cubic B-spline interpolation:
\begin{equation}
g(d; \psi) = \sum_{m=1}^{J} \psi_m B_m(d),
\end{equation}
with $B_m(d)$ denoting the cubic B-spline basis functions, $\psi = (\psi_1, \ldots, \psi_J)$ the learnable control parameters, and $J$ the number of spline knots.

\subsubsection{Validation-Based Threshold Optimization}
\label{subsec:threshold_optimization}
Given a validation set with GT annotations, the optimal spline parameters $\psi^*$ are obtained by maximizing a detection performance metric $\mathcal{M}$:
\begin{equation}
\psi^* = \arg\max_{\psi} \mathcal{M}\left(\{i : s_i \geq \tau(d_i; \psi)\}\right),
\end{equation}
where $s_i$ denotes the confidence score for detection $i$ at depth $d_i$, and $\{i : s_i \geq \tau(d_i; \psi)\}$ is the set of detections retained after applying depth-dependent thresholding.

The spline-based method makes no a priori assumptions about the form of the depth-threshold relationship. Instead, it adapts to the specific characteristics of the detector and dataset. If no depth-dependent pattern exists in the data, the optimization will converge to an approximately constant curve. This makes DCT robust to cases where depth-aware adjustment provides no benefit. Optimizing the score thresholds does not require retraining and can thus be integrated into any pre-trained detector.

\subsubsection{Optimization Framework}
\label{subsec:optimization_framework}
For a given reference threshold $\tau_0$, the objective $\mathcal{M}$ is to maximize the number of true positives with a penalty for excessive false positives. The penalty term activates only when false positives exceed $(1+\epsilon)$ times the baseline count, allowing the optimization to focus on increasing recall while maintaining precision:
\begin{equation}
\begin{aligned}
\mathcal{M}(\psi) = &\text{TP}(\psi) \\&- \gamma \, \max\!\Bigl(0,\, \text{FP}(\psi) - \text{FP}_{\text{static}}(1 + \epsilon)\Bigr),
\end{aligned}
\end{equation}
subject to:
\begin{equation}
\begin{aligned}
\text{FP}(\psi) &\leq \text{FP}_{\text{static}}(1 + \epsilon), \\
\psi_m \in [\rho, 1], \quad \forall m \in \{1, \ldots, J\}.
\end{aligned}
\end{equation}
Here, $\text{TP}$ and $\text{FP}$ denote true and false positive counts under the adaptive threshold, $\text{FP}_{\text{static}}$ is the baseline false positive count with uniform threshold $\tau_0$, $\epsilon$ controls the allowable false positive tolerance, $\gamma$ is a regularization parameter penalizing excessive false positives, and $\rho$ is the minimum admissible threshold (set to the detector's minimum meaningful confidence score).

\subsubsection{Depth Estimation and Integration}
For each predicted bounding box $\hat{b}_i = (x_1, y_1, x_2, y_2)$, we compute the normalized depth as:
\begin{equation} \label{eq:dnormdct}
    d_{i,\text{norm}} = 1 - \frac{\frac{1}{|\mathcal{R}_i|} \sum_{(h,w) \in \mathcal{R}_i} D_n(h,w) - \min(D_n)}{\max(D_n) - \min(D_n)},
\end{equation}
where $\mathcal{R}_i = \{(h,w) : x_1 \leq w \leq x_2, y_1 \leq h \leq y_2\}$ denotes the pixel region within the bounding box, $D_n(h,w)$ is the depth value at pixel $(h,w)$ in image $n$.

\subsubsection{Lookup Table Construction and Deployment}
Since the objective is non-differentiable, we use Bayesian optimization as a black-box optimizer. The optimization process (Algorithm~\ref{alg:lookup_generation}) generates a lookup table $\mathcal{T}$ mapping reference thresholds to optimized spline parameters:
\begin{equation}
\mathcal{T} = \{(\tau_{0,r}, \psi^*_r)\}_{r=1}^{R},
\end{equation}
where $R$ is the number of evaluated threshold values.

During inference (Algorithm~\ref{alg:DCT}), consider a set of candidate detections $\{(\hat{b}_i, s_i, d_{i,\text{norm}})\}_{i=1}^{N_{\text{det}}}$, where each detection $i$ consists of a predicted bounding box $\hat{b}_i$, its associated confidence score $s_i$, and its normalized depth $d_{i,\text{norm}}$ computed via Equation~\ref{eq:dnormdct}. Given a reference threshold $\tau_0$, the filtering proceeds as detailed in Algorithm~\ref{alg:DCT}.

\begin{algorithm}[htb]
\caption{Lookup Table Construction}
\label{alg:lookup_generation}
\begin{algorithmic}[1]
\REQUIRE Validation set, reference thresholds $\{\tau_{0,r}\}_{r=1}^{R}$, number of spline knots $J$, penalty weight $\gamma$, FP tolerance $\epsilon$, minimum threshold $\rho$
\ENSURE Lookup table $\mathcal{T}$
\STATE $\mathcal{T} \leftarrow \emptyset$
\FOR{$r = 1$ to $R$}
    \STATE $\text{FP}_{\text{static}} \leftarrow$ false positives using uniform threshold $\tau_{0,r}$
    \STATE $\psi^*_r \leftarrow \arg\max_{\psi} \; \mathcal{M}(\psi)$ where
    \STATE \quad $\mathcal{M}(\psi) = \text{TP}(\psi) - \gamma \cdot \max(0, \text{FP}(\psi) - \text{FP}_{\text{static}}(1 + \epsilon))$
    \STATE \quad s.t. $\psi_m \in [\rho, 1], \; \forall m \in \{1, \ldots, J\}$
    \STATE $\mathcal{T} \leftarrow \mathcal{T} \cup \{(\tau_{0,r}, \psi^*_r)\}$
\ENDFOR
\RETURN $\mathcal{T}$
\end{algorithmic}
\end{algorithm}

\begin{algorithm}[htb]
\caption{Lookup Table Deployment}
\label{alg:DCT}
\begin{algorithmic}[1]
\REQUIRE Detections $\{(\hat{b}_i, s_i)\}_{i=1}^{N_{\text{det}}}$, depth map $D_n \in \mathbb{R}^{H \times W}$, reference threshold $\tau_0$, lookup table $\mathcal{T}$, number of spline knots $J$
\ENSURE Filtered detections
\STATE $\psi^* \leftarrow \text{Lookup}(\tau_0, \mathcal{T})$
\STATE Filtered $\leftarrow \emptyset$
\FOR{$i = 1$ to $N_{\text{det}}$}
    \STATE $\mathcal{R}_i \leftarrow$ pixels within bounding box $\hat{b}_i$
    \STATE $d_{i,\text{norm}} \leftarrow 1 - \dfrac{\frac{1}{|\mathcal{R}_i|} \sum_{(h,w) \in \mathcal{R}_i} D_n(h,w) - \min(D_n)}{\max(D_n) - \min(D_n)}$
    \STATE $\tau_i \leftarrow \text{clip}\left(\tau_0 - \sum_{m=1}^{J} \psi^*_m B_m(d_{i,\text{norm}}), 0, 1\right)$
    \IF{$s_i \geq \tau_i$}
        \STATE Filtered $\leftarrow$ Filtered $\cup \; \{\hat{b}_i\}$
    \ENDIF
\ENDFOR
\RETURN Filtered
\end{algorithmic}
\end{algorithm}

\subsection{Hyperparameters}
\label{subsec:hyperparameters}
DepthPrior introduces several hyperparameters that control the strength of depth-aware adjustments. Ablations are in Supplementary Material Section~H. For DLW, $\alpha$ controls reweighting strength; we recommend $\alpha = 1.0$. For DLS, $\beta$ sets the close/distant boundary; $\lambda_{\text{distant}} > \lambda_{\text{close}}$ compensates for training bias. Defaults: $\beta = 0.5$, $\lambda_{\text{close}} = 1.0$, $\lambda_{\text{distant}} = 2.0$. For DCT,  $J$ controls the flexibility of the threshold curve. Defaults: $J = 10$ knots over $[0, 0.9]$, $\epsilon = 0.1$, $\gamma=1000$, minimum threshold $\rho = 0.1$.

\subsection{Unified Perspective}
\label{subsec:unified_perspective}
Under the assumption that weighting functions are parameter-independent ($\frac{\partial w(d)}{\partial \theta} = 0$ and $\frac{\partial \lambda_k}{\partial \theta} = 0$), both DLW and DLS reweight gradients: $\nabla_\theta \mathcal{L}_{\text{DLW}} = w(d) \cdot \nabla_\theta \mathcal{L}_{\text{base}}$ and $\nabla_\theta \mathcal{L}_{\text{DLS}} = \sum_k \mathbbm{1}_{d \in S_k} \lambda_k \cdot \nabla_\theta \mathcal{L}_{\text{base}}$. DCT complements by correcting residual depth-dependent confidence biases at inference. Together, DepthPrior addresses depth-induced heteroscedasticity throughout the pipeline without architectural modifications.
% %%%%%%%%% EXPERIMENTS
\section{Experiments} \label{sec:experiments}
We empirically investigate the depth-detection relationship (Section~\ref{sec:depth_detection_analysis}), then present experiments on depth-aware training (Section~\ref{sec:depth_aware_training}), inference (Section~\ref{sec:dct_results}), and their combination (Section~\ref{sec:dwl_dct}). Experiments are conducted with the following settings.

\subsection{Implementation Details} \label{sec:implementation_details}
\textbf{Detectors.} We evaluate EfficientDet-D0~\cite{tan2020efficientdet} (anchor-based, filtering after Non-Maximum Suppression (NMS)) and YOLOv11s~\cite{Jocher_Ultralytics_YOLO_2023} (anchor-free, filtering before NMS). Architecture details in Supplementary Material Section~D.

\textbf{Datasets.} KITTI~\cite{geiger2012we} (automotive), VisDrone~\cite{du2019visdrone} (aerial), SUN RGB-D~\cite{song2015sun} (indoor), and MS COCO~\cite{lin2014microsoft} (general photography). These exhibit distinct depth distributions: KITTI far-field concentration, VisDrone mid-range, SUN RGB-D linear growth, MS COCO bimodal. Details in Supplementary Material Section~B.

\textbf{Training.} EfficientDet: 200 epochs, batch 8, input resolution $512\times256$. YOLOv11: 100 epochs, batch 64, input resolution $640\times640$. Data augmentation is disabled to avoid complications with depth maps. Both models use pre-trained weights from MS COCO~\cite{lin2014microsoft}, except for experiments on MS COCO, where we train from scratch. For DCT, training uses $10\%$ of the datasets in order to evaluate generalization capabilities on the remaining held-out $90\%$ (max $10,000$ images). Validation sets comprise $20\%$ of KITTI and SUN RGB-D (randomly selected) and default validation splits for VisDrone and MS-COCO. We select an IoU threshold of $0.5$ for matching predictions to GT annotations. Results represent averages across three independent training runs. 

\textbf{Depth Estimation.} For depth estimation, we use Depth Anything~\cite{yang2024depth} with default hyperparameters.

\textbf{Evaluation Metrics.} We follow standard MS COCO evaluation~\cite{lin2014microsoft}, reporting mean Average Precision (mAP) averaged across IoU thresholds $[0.5\!:\!0.05\!:\!0.95]$ and mAP$_{50}$ at $\text{IoU}=0.5$. Objects are sorted by pixel area into small (S, area$<32^2$), medium (M, $32^2\leq$area$<96^2$), and large (L, area$\geq96^2$), with corresponding precision (mAP$_{S/M/L}$) and recall (mAR$_{S/M/L}$) metrics. DCT experiments additionally track true detections (TD) and extra detections (ED) to measure precision-recall trade-offs from threshold adjustments. Extra detections are either false detections or detections of unlabeled objects (missing GT annotations).

\subsection{Depth-Detection Relationship Analysis}
\label{sec:depth_detection_analysis}
We first investigate the relationship between object depth and detector behavior.

\subsubsection{Threshold-Independent Error Analysis}
Figure~\ref{fig:error_distribution_main} compares GT depth distributions against missing detections (MD) for EfficientDet across all datasets. MD rates are consistently higher at depth $d > 0.7$ across all datasets. Since MD are objects the detector misses entirely, i.e., the detector produces no detection at all or produces a detection with a confidence score below $0.1$. Thus, MD cannot be recovered by adjusting the confidence threshold in post-processing during inference. This motivates our training-time interventions DLW and DLS.

\begin{figure}[htbp]
    \centering
        \includegraphics[width=\linewidth]{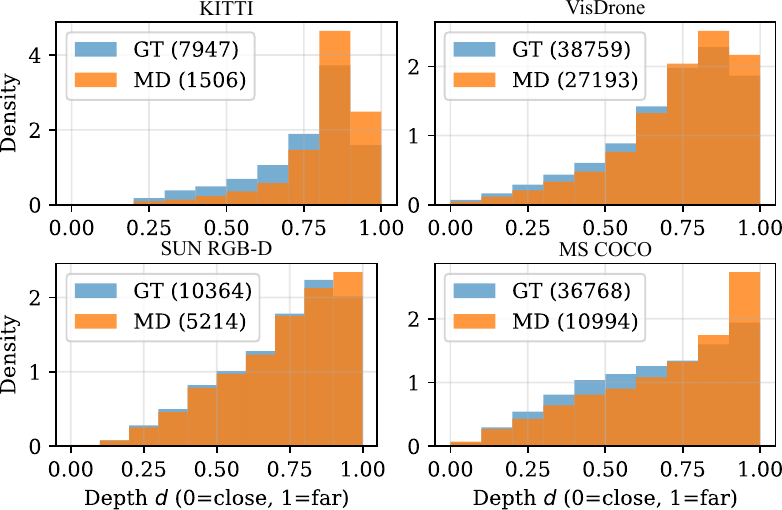}
    \caption{Depth distribution of all GT objects (blue) vs.~MD (orange) for EfficientDet on validation data. Object counts shown in parentheses.}
    \label{fig:error_distribution_main}
\end{figure}

\subsubsection{Threshold-Dependent Error Analysis}
With training-time interventions motivated, we now examine inference-time interventions, i.e., whether post-processing threshold adjustments can reduce detection errors. Figure~\ref{fig:threshold_analysis_main} compares MD and ED across depths at $\tau_0=0.4$ and $\tau_0=0.9$ for EfficientDet on the validation set. The varying MD-ED balance demonstrates that uniform thresholds cannot optimize both simultaneously. VisDrone and MS COCO show favorable MD-ED ratios at mid-range depths ($0.1-0.6$). Hence, lowering thresholds recovers MD without proportional increases in ED. However, at far depths ($> 0.8$), threshold reduction may substantially inflate ED. This pattern holds across threshold settings, though with different magnitudes. At $\tau_0 = 0.9$, more MD reduction potential exists in far ranges ($0.8-1.0$) of KITTI because lowering the threshold still maintains sufficient confidence to suppress ED. Meanwhile, at $\tau_0 = 0.4$, the same threshold reduction risks accepting low-confidence ED, as the model uncertainty increases. This threshold-dependent behavior varies systematically with depth, motivating DCT's depth-specific thresholding curves. Cross-architecture validation in Supplementary Material Section~E confirms these patterns.
\begin{figure}[htbp]
    \centering
    \includegraphics[width=\linewidth]{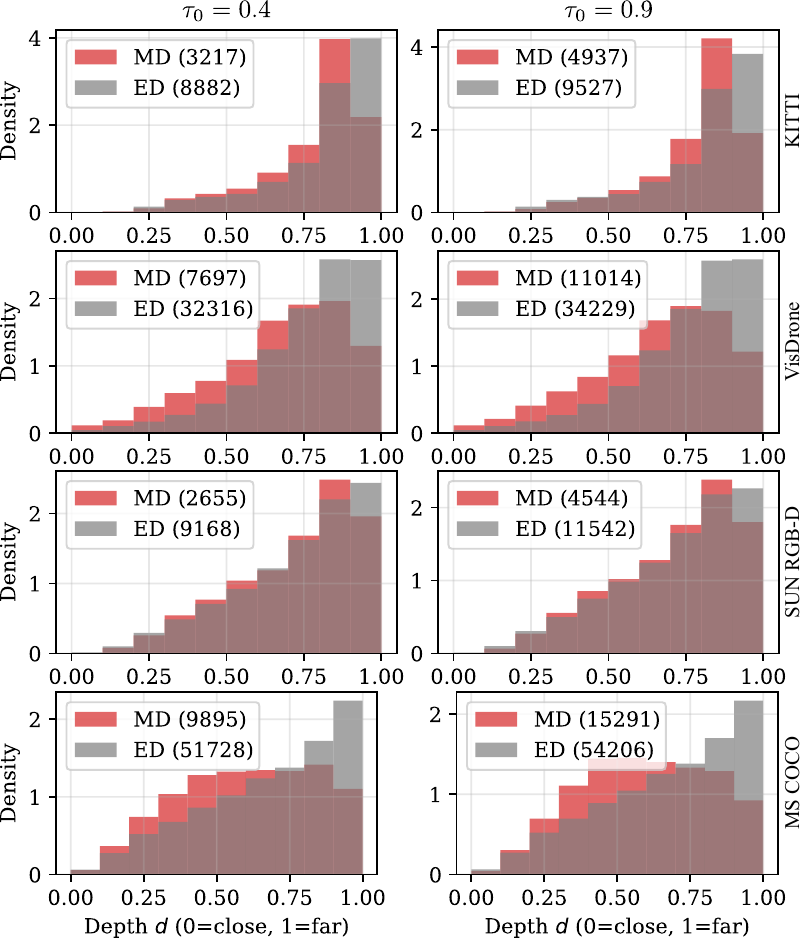}
    \caption{Depth-dependent error distributions (MD, red and ED, gray) at confidence thresholds $\tau_0 = 0.4$ (left) and $\tau_0 = 0.9$ (right) for EfficientDet on the validation set.}
    \label{fig:threshold_analysis_main}
\end{figure}

\subsubsection{Confidence-Depth Relationship}
While the threshold-dependent analysis reveals depth-specific MD-ED trade-offs, it does not account for the score distributions underlying these errors. Detections at the same depth may exhibit different confidence scores, potentially enabling selective recovery based on confidence-depth patterns. 

\begin{figure}[htbp]
    \centering
    \includegraphics[width=\linewidth]{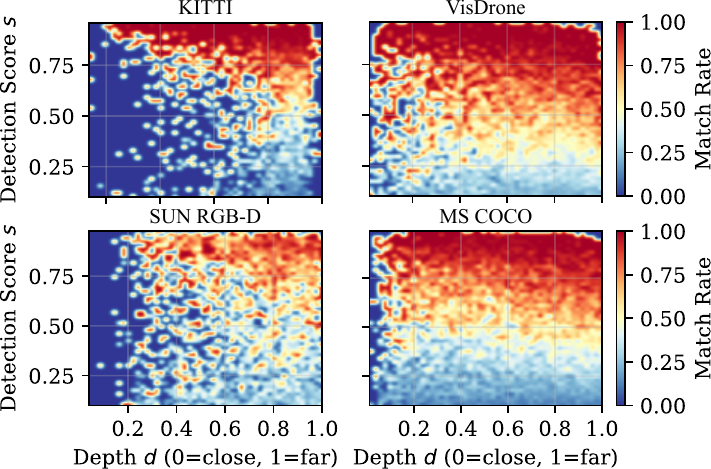}
    \caption{Match rate heatmaps for YOLOv11 on the validation set. Color indicates fraction of detections matching GT.}
    \label{fig:confidence_depth_main}
\end{figure}
Figure~\ref{fig:confidence_depth_main} shows match rates, the percentage of all detections that are actually matched to a GT, across confidence-depth space for YOLOv11 on the validation set. The heatmaps reveal systematic confidence degradation for distant objects across all datasets. High match rates (red/yellow regions) extend into low-confidence ranges (scores $0.3-0.6$) at far depths ($> 0.7$), indicating that many correctly detected distant objects receive inappropriately low confidence scores. YOLOv11 exhibits concentration of matches at high scores ($> 0.8$), but notable correctly detected objects exist at lower scores, especially at far depths. This confidence-depth structure enables selective MD recovery. For a user selecting $\tau_0 = 0.9$ on VisDrone, DCT can safely reduce the threshold for distant objects. The match rate remains high ($> 0.8$) down to a score of $0.5$ at far depths, indicating that threshold relaxation recovers MD without proportional ED inflation. In contrast, very low scores ($< 0.2$) show minimal match rates across all depths for both detectors, establishing a lower bound below which threshold reduction adds primarily ED. Cross-architecture validation in Supplementary Material Section~E confirms these patterns generalize to EfficientDet.

\subsection{Depth-Aware Training}\label{sec:depth_aware_training}
We establish with our theoretical considerations in Section~\ref{sec:dd_rel_methods} and empirical investigation in \ref{sec:depth_detection_analysis} the base for our DepthPrior framework. We now provide empirical evidence that depth-informed supervision helps mitigate the depth-induced bias in object detectors.

These baselines test whether (1) architectural depth integration helps 2D detection, (2) joint depth prediction improves performance, and (3) our specific weighting formulations are necessary beyond simple depth-proportional weighting. Tables~\ref{tab:depth_training_kitti}--\ref{tab:depth_training_mscoco} present the training results. 

\begin{table*}[htp]
    \centering
    \caption{Depth-aware training on KITTI with EfficientDet. DLW: $\alpha=1.0$. DLS: $\lambda_\text{close}=1$, $\lambda_\text{distant}=5$, $\beta = 75\%$.}
    \label{tab:depth_training_kitti}
    \resizebox{\textwidth}{!}{%
    \begin{tabular}{l|ccccc|cccc}
         \toprule
         \textbf{Method} & \textbf{mAP} & \textbf{mAP}$_{50}$ & \textbf{mAP}$_S$ & \textbf{mAP}$_M$ & \textbf{mAP}$_L$ & \textbf{mAR} & \textbf{mAR}$_S$ & \textbf{mAR}$_M$ & \textbf{mAR}$_L$ \\
        \midrule
        Baseline & $53.5_{\pm0.3}$ & $81.2_{\pm0.3}$ & $19.7_{\pm0.6}$ & $53.7_{\pm0.1}$ & $72.6_{\pm0.9}$ & $66.8_{\pm0.5}$ & $38.5_{\pm0.5}$ & $66.5_{\pm0.4}$ & $80.7_{\pm0.8}$ \\
        \midrule
        BW & $52.4_{\pm0.1}$ & $80.8_{\pm0.2}$ & $19.2_{\pm0.6}$ & $52.0_{\pm0.1}$ & $72.0_{\pm1.0}$ & $66.1_{\pm0.1}$ & $38.1_{\pm1.1}$ & $66.0_{\pm0.1}$ & $79.9_{\pm0.7}$ \\
        IW & $53.2_{\pm0.7}$ & $81.0_{\pm0.8}$ & $20.9_{\pm1.3}$ & $53.4_{\pm0.6}$ & $72.5_{\pm1.2}$ & $66.6_{\pm0.3}$ & $39.1_{\pm0.9}$ & $66.1_{\pm0.4}$ & $80.8_{\pm0.5}$ \\
        PD & $3.1_{\pm2.1}$ & $9.7_{\pm6.1}$ & $0.0_{\pm0.0}$ & $1.6_{\pm1.3}$ & $7.6_{\pm4.9}$ & $8.3_{\pm4.6}$ & $0.0_{\pm0.0}$ & $5.8_{\pm3.6}$ & $18.2_{\pm9.4}$ \\
        L & $16.3_{\pm1.9}$ & $31.8_{\pm3.5}$ & $2.4_{\pm1.0}$ & $13.1_{\pm2.7}$ & $31.3_{\pm2.7}$ & $28.2_{\pm3.0}$ & $9.7_{\pm3.1}$ & $26.5_{\pm3.5}$ & $42.9_{\pm3.2}$ \\
        CL & $15.0_{\pm1.7}$ & $28.6_{\pm2.7}$ & $1.1_{\pm0.2}$ & $12.3_{\pm1.7}$ & $29.2_{\pm2.7}$ & $25.7_{\pm2.2}$ & $6.7_{\pm2.6}$ & $24.8_{\pm2.2}$ & $39.4_{\pm2.0}$ \\
        \midrule
        \textbf{DLW} & $54.7_{\pm0.2}$ & $82.6_{\pm0.3}$ & $23.7_{\pm1.2}$ & $54.1_{\pm0.6}$ & $\mathbf{73.3_{\pm0.3}}$ & $67.6_{\pm0.1}$ & $41.3_{\pm1.8}$ & $67.3_{\pm0.4}$ & $\mathbf{81.1_{\pm0.3}}$ \\
        \textbf{DLS} & $\mathbf{55.1_{\pm0.5}}$ & $\mathbf{82.9_{\pm0.5}}$ & $\mathbf{28.4_{\pm0.5}}$ & $\mathbf{55.7_{\pm0.6}}$ & $71.3_{\pm0.2}$ & $\mathbf{68.3_{\pm0.2}}$ & $\mathbf{45.0_{\pm0.7}}$ & $\mathbf{68.6_{\pm0.1}}$ & $79.3_{\pm0.4}$ \\
         \bottomrule
    \end{tabular}%
    }
\end{table*}
\begin{table*}[htp]
    \centering
    \caption{Depth-aware training on SUN RGB-D with EfficientDet. DLW: $\alpha=1.0$. DLS: $\lambda_\text{close}=1$, $\lambda_\text{distant}=5$, $\beta = 25\%$.}
    \label{tab:depth_training_sunrgbd}
    \resizebox{\textwidth}{!}{%
    \begin{tabular}{l|ccccc|cccc}
         \toprule
         \textbf{Method} & \textbf{mAP} & \textbf{mAP}$_{50}$ & \textbf{mAP}$_S$ & \textbf{mAP}$_M$ & \textbf{mAP}$_L$ & \textbf{mAR} & \textbf{mAR}$_S$ & \textbf{mAR}$_M$ & \textbf{mAR}$_L$ \\
        \midrule
        Baseline & $17.6_{\pm0.3}$ & $31.3_{\pm0.3}$ & $\mathbf{2.0_{\pm0.2}}$ & $8.9_{\pm0.7}$ & $24.7_{\pm0.3}$ & $32.0_{\pm0.3}$ & $\mathbf{3.8_{\pm0.3}}$ & $20.8_{\pm0.6}$ & $43.7_{\pm1.8}$ \\
        \midrule
        BW & $18.5_{\pm0.2}$ & $33.2_{\pm0.2}$ & $1.3_{\pm0.3}$ & $9.3_{\pm0.3}$ & $26.6_{\pm0.4}$ & $\mathbf{36.4_{\pm0.6}}$ & $3.2_{\pm0.3}$ & $\mathbf{25.0_{\pm0.9}}$ & $\mathbf{47.3_{\pm0.9}}$ \\
        IW & $18.3_{\pm0.4}$ & $32.3_{\pm0.1}$ & $1.3_{\pm0.3}$ & $9.5_{\pm0.3}$ & $25.8_{\pm0.6}$ & $35.3_{\pm0.4}$ & $3.0_{\pm0.3}$ & $24.2_{\pm0.7}$ & $45.6_{\pm0.8}$ \\
        PD & $0.5_{\pm0.3}$ & $1.9_{\pm0.9}$ & $0.0_{\pm0.1}$ & $0.1_{\pm0.1}$ & $0.7_{\pm0.4}$ & $2.0_{\pm1.3}$ & $0.2_{\pm0.3}$ & $0.8_{\pm0.6}$ & $2.9_{\pm1.9}$ \\
        L & $12.7_{\pm0.5}$ & $23.2_{\pm1.0}$ & $0.5_{\pm0.2}$ & $5.5_{\pm0.3}$ & $18.6_{\pm0.7}$ & $27.7_{\pm1.2}$ & $1.4_{\pm0.4}$ & $16.5_{\pm1.1}$ & $37.2_{\pm0.4}$ \\
        CL & $10.2_{\pm0.2}$ & $19.7_{\pm0.5}$ & $0.5_{\pm0.1}$ & $4.3_{\pm0.5}$ & $14.9_{\pm0.3}$ & $24.4_{\pm1.0}$ & $0.9_{\pm0.2}$ & $13.3_{\pm0.1}$ & $33.6_{\pm1.9}$ \\
        \midrule
        \textbf{DLW} & $\mathbf{19.1_{\pm0.1}}$ & $\mathbf{33.6_{\pm0.2}}$ & $1.5_{\pm0.2}$ & $\mathbf{9.8_{\pm0.3}}$ & $\mathbf{27.0_{\pm0.1}}$ & $35.3_{\pm0.2}$ & $3.5_{\pm0.2}$ & $24.3_{\pm0.4}$ & $44.6_{\pm0.3}$ \\
        \textbf{DLS} & $18.7_{\pm0.4}$ & $32.8_{\pm0.7}$ & $1.6_{\pm0.3}$ & $9.4_{\pm0.5}$ & $26.6_{\pm0.7}$ & $35.2_{\pm0.7}$ & $3.5_{\pm0.3}$ & $24.2_{\pm0.8}$ & $45.1_{\pm1.1}$ \\
         \bottomrule
    \end{tabular}%
    }
\end{table*}
\begin{table*}[htp]
    \centering
    \caption{Depth-aware training on VisDrone with EfficientDet. PD omitted due to consistent poor performance on prior datasets. DLW: $\alpha=1.0$. DLS: $\lambda_\text{close}=1$, $\lambda_\text{distant}=5$, $\beta = 15\%$.}
    \label{tab:depth_training_visdrone}
    \resizebox{\textwidth}{!}{%
    \begin{tabular}{l|ccccc|cccc}
         \toprule
         \textbf{Method} & \textbf{mAP} & \textbf{mAP}$_{50}$ & \textbf{mAP}$_S$ & \textbf{mAP}$_M$ & \textbf{mAP}$_L$ & \textbf{mAR} & \textbf{mAR}$_S$ & \textbf{mAR}$_M$ & \textbf{mAR}$_L$ \\
        \midrule
        Baseline  & $6.2_{\pm0.1}$ & $11.2_{\pm0.2}$ & $0.9_{\pm0.0}$ & $9.5_{\pm0.2}$ & $33.4_{\pm0.9}$ & $9.3_{\pm0.2}$ & $2.7_{\pm0.1}$ & $15.2_{\pm0.4}$ & $40.8_{\pm0.5}$ \\
        \midrule
        BW & $6.1_{\pm0.1}$ & $11.0_{\pm0.2}$ & $0.9_{\pm0.0}$ & $9.1_{\pm0.3}$ & $\mathbf{35.7_{\pm1.7}}$ & $9.1_{\pm0.0}$ & $2.5_{\pm0.1}$ & $14.7_{\pm0.2}$ & $\mathbf{45.4_{\pm1.2}}$ \\
        IW & $6.1_{\pm0.1}$ & $11.0_{\pm0.1}$ & $0.9_{\pm0.1}$ & $9.3_{\pm0.4}$ & $34.2_{\pm0.8}$ & $9.2_{\pm0.1}$ & $2.6_{\pm0.1}$ & $14.8_{\pm0.2}$ & $43.3_{\pm1.0}$ \\
        L & $3.1_{\pm0.3}$ & $5.4_{\pm0.6}$ & $0.4_{\pm0.1}$ & $3.9_{\pm0.4}$ & $19.4_{\pm0.2}$ & $4.7_{\pm0.5}$ & $1.4_{\pm0.2}$ & $6.3_{\pm0.6}$ & $24.3_{\pm0.8}$ \\
        CL & $3.7_{\pm0.1}$ & $6.8_{\pm0.3}$ & $0.6_{\pm0.1}$ & $5.1_{\pm0.2}$ & $21.8_{\pm1.3}$ & $6.0_{\pm0.3}$ & $1.7_{\pm0.2}$ & $9.2_{\pm0.6}$ & $30.8_{\pm2.7}$ \\
        \midrule
        \textbf{DLW} & $6.5_{\pm0.1}$ & $\mathbf{12.1_{\pm0.3}}$ & $\mathbf{1.1_{\pm0.1}}$ & $\mathbf{10.0_{\pm0.1}}$ & $33.2_{\pm1.0}$ & $\mathbf{9.7_{\pm0.1}}$ & $\mathbf{2.9_{\pm0.0}}$ & $15.7_{\pm0.1}$ & $42.1_{\pm1.3}$ \\
        \textbf{DLS} & $\mathbf{6.6_{\pm0.2}}$ & $11.9_{\pm0.2}$ & $1.0_{\pm0.1}$ & $\mathbf{10.0_{\pm0.3}}$ & $33.3_{\pm1.2}$ & $\mathbf{9.7_{\pm0.2}}$ & $2.8_{\pm0.1}$ & $\mathbf{15.8_{\pm0.5}}$ & $40.8_{\pm0.2}$ \\
         \bottomrule
    \end{tabular}%
    }
\end{table*}
\begin{table*}[htp]
    \centering
     \caption{Depth-aware training on MS COCO with EfficientDet. Only DepthPrior methods evaluated against the vanilla baseline due to computational cost and established poor performance of baselines. DLW: $\alpha=1.0$. DLS: $\lambda_\text{close}=1$, $\lambda_\text{distant}=5$, $\beta = 25\%$.}
    \label{tab:depth_training_mscoco}
    \resizebox{\textwidth}{!}{%
    \begin{tabular}{l|ccccc|cccc}
         \toprule
         \textbf{Method} & \textbf{mAP} & \textbf{mAP}$_{50}$ & \textbf{mAP}$_S$ & \textbf{mAP}$_M$ & \textbf{mAP}$_L$ & \textbf{mAR} & \textbf{mAR}$_S$ & \textbf{mAR}$_M$ & \textbf{mAR}$_L$ \\
        \midrule
        Baseline  & $25.3_{\pm0.1}$ & $39.7_{\pm0.1}$ & $1.3_{\pm0.1}$ & $18.0_{\pm0.1}$ & $43.8_{\pm0.2}$ & $35.7_{\pm0.2}$ & $3.9_{\pm0.2}$ & $32.1_{\pm0.4}$ & $57.4_{\pm0.2}$ \\
        \midrule
        \textbf{DLW} & $\mathbf{25.7_{\pm0.4}}$ & $\mathbf{40.2_{\pm0.5}}$ & $1.5_{\pm0.2}$ & $\mathbf{18.8_{\pm0.3}}$ & $\mathbf{44.4_{\pm0.4}}$ & $\mathbf{36.4_{\pm0.3}}$ & $\mathbf{4.4_{\pm0.3}}$ & $\mathbf{33.4_{\pm0.7}}$ & $\mathbf{58.0_{\pm0.5}}$\\
        \textbf{DLS}  & $25.4_{\pm0.3}$ & $40.1_{\pm0.3}$ & $\mathbf{1.6_{\pm0.2}}$ & $18.7_{\pm0.4}$ & $43.5_{\pm0.4}$ & $36.3_{\pm0.4}$ & $\mathbf{4.4_{\pm0.3}}$ & $\mathbf{33.4_{\pm0.5}}$ & $57.4_{\pm0.3}$\\ 
         \bottomrule
    \end{tabular}%
     }
\end{table*}

\textbf{Architectural Integration.} Architecturally-integrated depth methods (L, CL, PD) fail across all datasets. On KITTI, they achieve $16.3\%$, $15.0\%$, and $3.1\%$ mAP vs.\ $53.5\%$ baseline. Similar patterns hold on SUN RGB-D ($12.7\%$, $10.2\%$, $0.5\%$ vs.\ $17.6\%$) and VisDrone ($3.1\%$, $3.7\%$ vs.\ $6.2\%$). Architectural depth integration creates optimization challenges and negative transfer that outweigh geometric benefits. In contrast, DLW and DLS operate at the loss level, requiring only depth maps alongside GT labels.

\textbf{Simple Weighting.} Naive depth-proportional weighting (BW, IW) shows inconsistent results. BW reduces performance on KITTI ($-1.1\%$ mAP) and VisDrone ($-0.1\%$), but slightly increases it on SUN RGB-D ($+0.9\%$). IW performs slightly better but remains unreliable: $-0.3\%$ (KITTI), $+0.7\%$ (SUN RGB-D), $-0.1\%$ (VisDrone). 

Adding our exponential transformation improves performance. BW achieves $54.0\%$ mAP on KITTI ($+0.5\%$) and IW reaches $54.2\%$ mAP ($+0.7\%$). These results approach DLW (KITTI: $54.7\%$). However, the gap between averaging-based approaches and per-object weighting remains consistent, demonstrating that per-object granularity matters alongside non-linear weighting. 

\textbf{Our Methods.} Both DLW and DLS improve performance across all datasets. KITTI shows largest improvements (Table~\ref{tab:depth_training_kitti}), consistent with its far-field concentration and strong size-depth correlation ($r=-0.741^{***}$, Supplementary Material Section~C). DLW: $+1.2\%$ mAP ($54.7\%$), $+4.0\%$ mAP$_S$ ($23.7\%$). DLS: $+1.6\%$ mAP ($55.1\%$), $+8.7\%$ mAP$_S$ ($28.4\%$), a 44\% relative improvement in small object precision. SUN RGB-D (Table~\ref{tab:depth_training_sunrgbd}): DLW $+1.5\%$ mAP ($19.1\%$), DLS $+1.1\%$ ($18.7\%$). The bounded indoor depth range leaves less room for precision improvement, yet small object recall improves ($+3.5\%$ mAR$_S$). VisDrone (Table~\ref{tab:depth_training_visdrone}): Despite weaker size-depth correlation ($r=-0.275^{***}$) from aerial perspective, DLW achieves $+0.3\%$ mAP ($6.5\%$) and DLS $+0.4\%$ ($6.6\%$), with improvements on small/medium objects (mAP$_S$: $+0.2\%$, mAP$_M$: $+0.5\%$, mAR$_M$: $+0.5-0.6\%$). This indicates that depth-informed supervision is beneficial even when geometric relationships are less pronounced and detector performance is low. MS COCO (Table~\ref{tab:depth_training_mscoco}): DLW $+0.4\%$ mAP ($25.7\%$), DLS $+0.1\%$ ($25.4\%$), with $+0.2$--$0.3\%$ mAP$_S$ improvements despite weak correlation ($r=-0.286^{***}$).

Results across all benchmarks suggest depth-induced heteroscedasticity is a common challenge, and depth-informed supervision helps mitigate it. The selected configurations prioritize balanced performance. Ablations in Supplementary Material Section~H show that alternative hyperparameters enable stronger small object focus when needed. Note that COCO-style evaluation uses size bins (area thresholds) rather than depth. 

\textbf{DLW vs.\ DLS.} DLW and DLS exhibit complementary strengths. DLW's continuous weighting suits complex depth distributions, achieving higher overall mAP on 3/4 datasets (SUN RGB-D: $19.1\%$ vs.~$18.7\%$, VisDrone: $6.5\%$ vs.~$6.6\%$, MS COCO: $25.7\%$ vs.~$25.4\%$). DLS excels when clear depth boundaries exist (KITTI: $55.1\%$ vs.\ $54.7\%$), improving performance on small objects over DLW on 3/4 datasets (KITTI mAP$_S$ $28.4\%$ vs.~$23.7\%$, SUN RGB-D mAP$_S$ $1.6\%$ vs.~$1.5\%$, MS COCO mAP$_S$ $1.6\%$ vs.~$1.5\%$) at the cost of some large object performance (KITTI mAP$_L$: $71.3\%$ vs.~DLW:~$73.3\%$, mAR$_L$: $79.3\%$ vs. DLW:~$81.1\%$). This reflects the challenge in depth-informed training. Balancing limited model capacity toward difficult examples necessarily reduces emphasis on easier examples. Use DLW for balanced performance with minimal tuning and DLS when small/distant detection is mission-critical, clear depth intervals exists, and hyperparameter tuning is possible. 

\subsubsection{Cross-Architecture Validation: YOLOv11}
Table~\ref{tab:yolov11_kitti_main} validates DepthPrior on YOLOv11. DLW achieves $+1.0\%$ mAP ($49.4\%$ vs.~$48.4\%$ baseline), $+1.5\%$ over BW ($47.9\%$) and $+0.8\%$ over IW ($48.6\%$). DLS maximizes small objects with $34.4\%$ mAP$_S$ ($+3.4\%$ over baseline) and $41.2\%$ mAR$_S$ ($+3.1\%$ over baseline). Improvements are smaller than EfficientDet, suggesting anchor-free detectors already partially address small object challenges. This is also highlighted by YOLOv11 achieving lower baseline mAP than EfficientDet on KITTI ($48.4\%$ vs.~$53.5\%$), but stronger small object baseline ($31.0\%$ vs.~$19.7\%$ mAP$_S$). Despite these differences, both architectures exhibit the same depth-dependent degradation patterns. 

\begin{table*}[t]
    \centering
    \caption{Cross-architecture validation on KITTI using YOLOv11. DLW: $\alpha=1.0$. DLS: $\lambda_\text{close}=1$, $\lambda_\text{distant}=5$, $\beta = 75\%$.}
    \label{tab:yolov11_kitti_main}
    \resizebox{\textwidth}{!}{%
    \begin{tabular}{l|ccccc|cccc}
         \toprule
         \textbf{Method} & \textbf{mAP} & \textbf{mAP}$_{50}$ & \textbf{mAP}$_S$ & \textbf{mAP}$_M$ & \textbf{mAP}$_L$ & \textbf{mAR} & \textbf{mAR}$_S$ & \textbf{mAR}$_M$ & \textbf{mAR}$_L$ \\
        \midrule
        Baseline & $48.4_{\pm0.3}$ & $77.8_{\pm0.4}$ & $31.0_{\pm1.8}$ & $52.4_{\pm0.3}$ & $52.9_{\pm0.8}$ & $58.0_{\pm0.3}$ & $38.1_{\pm1.2}$ & $60.8_{\pm0.6}$ & $63.3_{\pm0.3}$ \\
        \midrule
        BW & $47.9_{\pm1.0}$ & $78.1_{\pm0.8}$ & $31.3_{\pm2.5}$ & $51.5_{\pm1.2}$ & $52.8_{\pm0.4}$ & $58.1_{\pm0.9}$ & $39.5_{\pm1.5}$ & $61.0_{\pm0.7}$ & $63.2_{\pm1.6}$ \\
        IW & $48.6_{\pm0.2}$ & $77.9_{\pm0.5}$ & $31.5_{\pm1.9}$ & $52.0_{\pm0.5}$ & $52.9_{\pm0.9}$ & $58.2_{\pm0.5}$ & $38.6_{\pm1.5}$ & $60.9_{\pm0.1}$ & $63.6_{\pm1.7}$ \\
        \midrule
        \textbf{DLW} & $\mathbf{49.4_{\pm0.4}}$ & $\mathbf{78.8_{\pm0.9}}$ & $33.3_{\pm0.2}$ & $\mathbf{53.3_{\pm0.6}}$ & $\mathbf{53.8_{\pm1.0}}$ & $\mathbf{59.2_{\pm0.1}}$ & $40.0_{\pm0.3}$ & $\mathbf{62.3_{\pm0.3}}$ & $\mathbf{63.8_{\pm0.7}}$ \\
        \textbf{DLS} & $47.4_{\pm0.8}$ & $77.8_{\pm1.5}$ & $\mathbf{34.4_{\pm0.5}}$ & $51.3_{\pm0.7}$ & $49.9_{\pm1.0}$ & $58.7_{\pm0.4}$ & $\mathbf{41.2_{\pm0.8}}$ & $61.4_{\pm0.9}$ & $63.2_{\pm0.6}$ \\
         \bottomrule
    \end{tabular}%
    }
\end{table*}
\subsection{Depth-Aware Inference}
\label{sec:dct_results}
We now demonstrate that depth-related inference patterns can be exploited through learned depth-aware thresholds. Baseline mAP performance across detectors on 10\% of the datasets varies from $3.7\%$ (EfficientDet on VisDrone) to $47.5\%$ (YOLOv11 on KITTI), testing DCT under diverse conditions (see Supplementary Material Section~G for details).

The learned threshold functions exhibit dataset-specific and architecture-specific characteristics that reflect the empirical patterns established in Section~\ref{sec:depth_detection_analysis}.  EfficientDet's dispersed confidence distribution allows a stronger threshold relaxation at far depths, while YOLOv11's peaked score distribution (concentrated at scores $>0.75$) results in more conservative adjustments. The magnitude and shape of adjustments vary across reference thresholds $\tau_0$, with higher thresholds allowing a stronger relaxation due to larger margins before ED inflation. Detailed curves, Pareto fronts, and per-threshold breakdowns are provided in the Supplementary Material Section~G.

Table~\ref{tab:dct_validation} quantifies DCT's performance on each of the validation data used for spline optimization. Aggregated across datasets and architectures: $+6,041$ TD with only $+184$ ED (33:1 ratio). KITTI shows strongest improvements for EfficientDet ($+1,427$ TD, $-21$ ED); VisDrone reduces ED ($-896$) while increasing TD. DCT on SUN RGB-D and MS COCO results in moderate improvements.

Architecture-specific differences reflect the adaptive nature of DCT's optimization. For YOLOv11, DCT shows larger total TD increase on VisDrone ($+1,778$) compared to EfficientDet ($+289$), though with correspondingly larger ED increases ($+302$ vs.~$-896$). Anchor-based EfficientDet's reliance on predefined anchor boxes makes it vulnerable to far-field detection failures where scale variation exceeds anchor coverage. For YOLOv11 on SUN RGB-D and MS COCO, DCT achieves stronger TD increase ($+579$ and $+926$ for each dataset, respectively) than EfficientDet ($+424$, $+474$) despite its concentration of matches at high confidence. This demonstrates that depth-aware thresholding remains effective even when the detector already performs well on most objects.

\begin{table}[htbp]
    \centering
\caption{DCT performance on validation set comparing static thresholding (TD, ED) against depth-aware thresholds (TD$^*$, ED$^*$) for EfficientDet (top) and YOLOv11 (bottom).}
    \label{tab:dct_validation}
\resizebox{\columnwidth}{!}{%
\begin{tabular}{l|rrrr|rr}
\toprule
Dataset & TD & ED & $\text{TD}^*$ & $\text{ED}^*$ & $\Delta_{TD}\uparrow$ & $\Delta_{ED}\downarrow$ \\
\midrule
KITTI & $36,706$ & $14,230$ & $38,133$ & $14,209$ & $1,427$ & $-21$ \\
VisDrone & $40,516$ & $59,888$ & $40,805$ & $58,992$ & $289$ & $-896$ \\
SUN RGB-D & $26,519$ & $26,982$ & $26,943$ & $27,178$ & $424$ & $196$ \\
MS COCO & $52,729$ & $70,508$ & $53,203$ & $70,648$ & $474$ & $140$ \\
\midrule
KITTI & $52,721$ & $8,991$ & $52,865$ & $9,026$ & $144$ & $35$ \\
VisDrone & $111,121$ & $35,227$ & $112,899$ & $35,529$ & $1,778$ & $302$ \\
SUN RGB-D & $26,227$ & $16,214$ & $26,806$ & $16,288$ & $579$ & $74$ \\
MS COCO & $139,062$ & $66,034$ & $139,988$ & $66,388$ & $926$ & $354$ \\
\midrule
$\sum$ & $485,601$ & $298,074$ & $491,642$ & $298,258$ & $6,041$ & $184$ \\
\bottomrule
\end{tabular}%
}
\end{table}

Following Section~\ref{sec:implementation_details}, we evaluate DCT on the held-out $90\%$ inference split that was excluded from both model training and threshold optimization. Table~\ref{tab:dct_generalization} shows effective generalization. The learned depth-threshold relationships transfer to unseen data, achieving $22,042$ additional TD with $2,412$ additional ED across all conditions (an approximately $9:1$ favorable ratio). DCT on VisDrone shows strong generalization. For YOLOv11, it achieves $+12,042$ TD on inference (compared to $+1,778$ on validation), while for EfficientDet, it maintains its unique pattern of ED reduction ($-4,664$ on inference vs.~$-896$ on validation). On KITTI, it demonstrates stable far-field recovery for EfficientDet maintaining an increase of $+1,681$ TD on inference vs.~$+1,427$ on validation. For YOLOv11, it shows more conservative improvements ($+359$ vs.~$+144$) consistent with its peaked confidence distribution. DCT on indoor (SUN RGB-D) and photographic (MS COCO) datasets exhibits moderate but consistent generalization. For YOLOv11, it maintains its relative improvements ($+1,861$, $+1,995$ TD) compared to EfficientDet ($+1,118$, $+1,042$ TD). This consistency confirms that the depth-dependent confidence degradation patterns represent properties of detector-dataset combinations rather than split-specific artifacts.
 
\begin{table}[htbp]
    \centering
\caption{DCT generalization on held-out inference data ($90\%$) unseen during training and threshold optimization. EfficientDet (top), YOLOv11 (bottom).}
    \label{tab:dct_generalization}
\resizebox{\columnwidth}{!}{%
\begin{tabular}{l|rrrr|rr}
\toprule
Dataset & TD & ED & $\text{TD}^*$ & $\text{ED}^*$ & $\Delta_{TD}\uparrow$ & $\Delta_{ED}\downarrow$ \\
\midrule
KITTI & $186,914$ & $42,489$ & $188,595$ & $43,309$ & $1,681$ & $820$ \\
VisDrone & $286,812$ & $625,026$ & $288,756$ & $620,362$ & $1,944$ & $-4,664$ \\
SUN RGB-D & $90,889$ & $99,878$ & $92,007$ & $100,935$ & $1,118$ & $1,057$ \\
MS COCO & $105,530$ & $136,724$ & $106,572$ & $137,185$ & $1,042$ & $461$ \\
\midrule
KITTI & $187,863$ & $32,723$ & $188,222$ & $32,918$ & $359$ & $195$ \\
VisDrone & $908,910$ & $269,394$ & $920,952$ & $272,484$ & $12,042$ & $3,090$ \\
SUN RGB-D & $91,035$ & $57,737$ & $92,896$ & $58,414$ & $1,861$ & $677$ \\
MS COCO & $299,803$ & $121,907$ & $301,798$ & $122,683$ & $1,995$ & $776$ \\
\midrule
$\sum$ & $2,157,756$ & $1,385,878$ & $2,179,798$ & $1,388,290$ & $22,042$ & $2,412$ \\
\bottomrule
\end{tabular}%
}
\end{table}

Figure~\ref{fig:dct_characteristics_score} analyzes which objects DCT recovers by visualizing size-depth distributions for KITTI and VisDrone. DCT's recovered objects differ not only geometrically but also in their score characteristics. KITTI's pattern supports the core theoretical hypothesis: far depth $\rightarrow$ small size (via perspective) $\rightarrow$ low confidence (via perceptual difficulty) $\rightarrow$ recovery opportunity (via learned depth-aware threshold). VisDrone's broader pattern confirms that DCT adapts to dataset-specific characteristics rather than assuming universal depth-difficulty relationships. As the threshold decreases on VisDrone and given YOLOv11's peaked score distribution, the recoverable detections shift toward closer objects.
\begin{figure}
        \centering
        \includegraphics[width=0.8\linewidth]{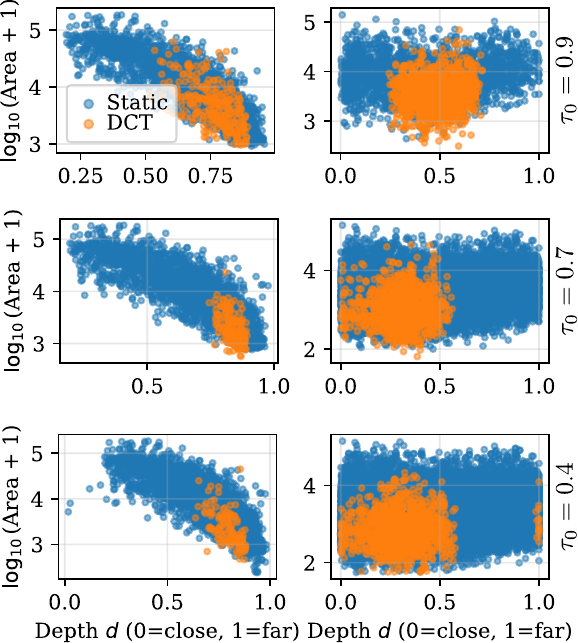}
        \caption{Static (blue) vs.\ DCT-recovered detections (orange) on inference data. Left: KITTI with EfficientDet. Right: VisDrone with YOLOv11.}
    \label{fig:dct_characteristics_score}
    \end{figure}

Figure~\ref{fig:dct_examples} provides qualitative examples of DCT for EfficientDet on KITTI and YOLOv11 on VisDrone. The visualizations show detection confidence scores and applied thresholds (original $\tau_0$ for static, depth-adjusted $\tau(d)$ for DCT), with green boxes indicating static detections and orange boxes showing additional DCT recoveries. Distant pedestrians at depth $>0.8$ in KITTI receive confidence scores $0.53-0.66$, falling below the static threshold $\tau_0=0.70$ but correctly identified as matches. DCT's learned curves reduce the threshold to $0.48-0.59$ at these depths, recovering the detections. VisDrone aerial scenes show broader-range recovery consistent with the quantitative analysis. Additional qualitative examples on MS COCO and SUN RGB-D are in Supplementary Material Section~G.

\begin{figure}[htp]
        \centering
    \includegraphics[width=1\linewidth]{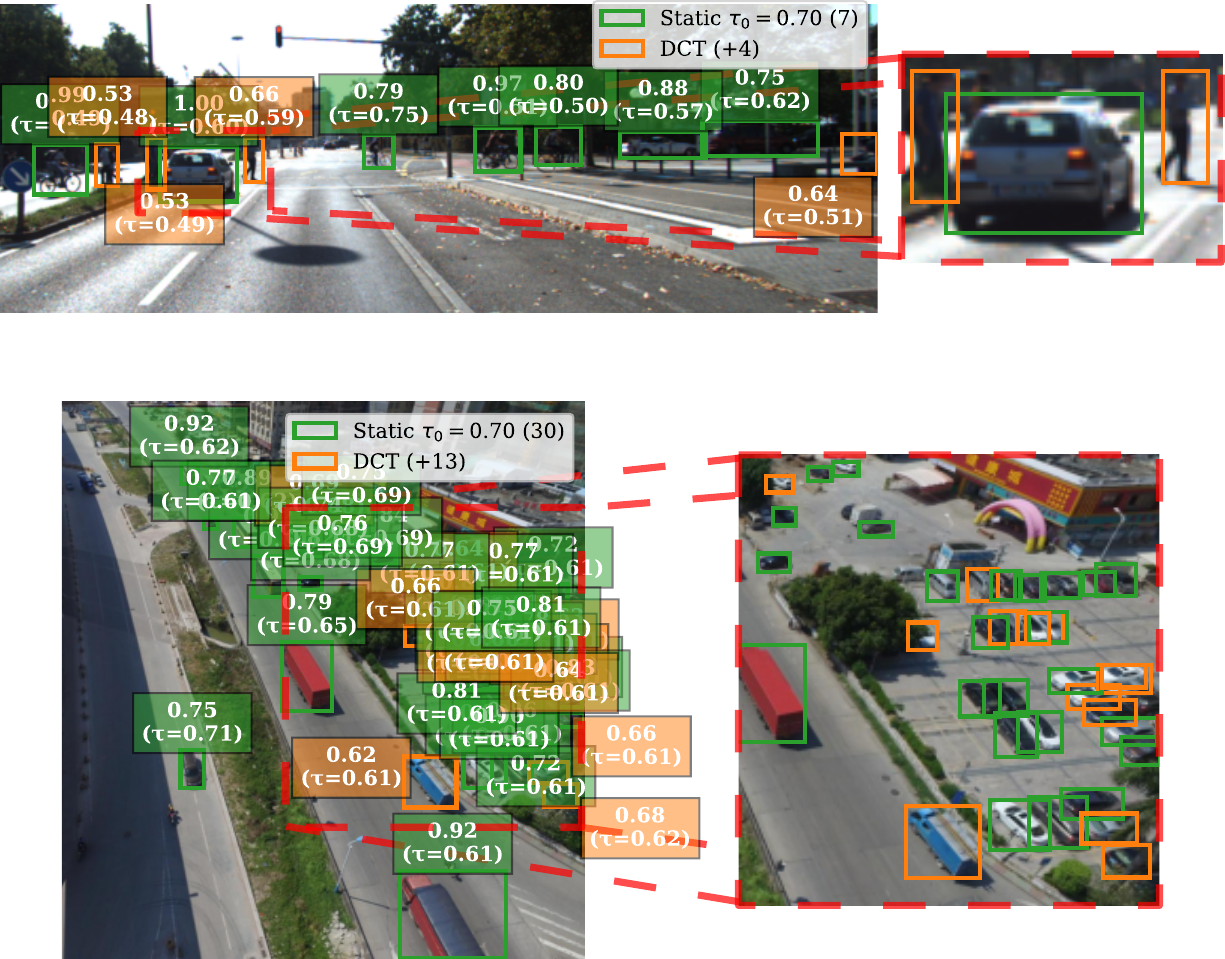}
    \caption{DCT examples on inference data. Top: EfficientDet on KITTI at $\tau_0=0.70$. Bottom: YOLOv11 on VisDrone at $\tau_0=0.70$. Green boxes: static detections. Orange boxes: DCT recoveries. Numbers show confidence scores and applied thresholds.}
    \label{fig:dct_examples}
\end{figure}

\subsection{DepthPrior}
\label{sec:dwl_dct}
Having established the effectiveness of depth-informed training (DLW, DLS) and depth-aware inference (DCT) separately, we now investigate their combination. When training on $10\%$ of KITTI (598 images), the baseline with vanilla training achieves $30.3\%$ mAP while DLW achieves $31.9\%$ mAP ($+1.6\%$). This demonstrates DepthPrior in low-data regimes. Small object improvements range from $7.1\%$ to $9.1\%$ mAP$_S$ ($+2.0\%$). 

Table~\ref{tab:dwl_dct_combined} combines DLW ($\alpha=1.0$) with DCT on KITTI. On the validation set, DLW+DCT achieves $+444$ TD at $\tau_0=0.9$ compared to DCT-only's $+230$ TD (+$93\%$) with minimal ED cost ($+1$ vs.~$0$). This pattern is however not consistent across mid- and low-range thresholds. At $\tau_0=0.9$ on the held-out set, DLW+DCT achieves $+1525$ TD with only $+16$ ED, outperforming DCT-only ($+809$ TD with $+54$ ED). This represents an $88\%$ improvement in TD recovery with $70\%$ reduction in ED inflation, showing that both components complement each other. The pattern persists across thresholds. At $\tau_0=0.5$, DLW+DCT recovers $+804$ TD vs.~DCT-only's $+226$ TD, while maintaining comparable ED control ($+66$ vs.~$+58$). DLW+DCT eliminates negative deltas present in DCT-only at $\tau_0\in\{0.4, 0.7\}$, indicating depth-aware training produces better-calibrated scores for DCT to exploit.

These results establish DepthPrior as a complete depth-aware detection system that maximizes performance through coordinated training and inference interventions, supporting the core hypothesis that depth-based supervision provides benefits without architectural modifications. 

\begin{table}[t]
\centering
\caption{Combined DLW+DCT performance on KITTI. $\Delta_{\text{TD}}$ and $\Delta_{\text{ED}}$ relative to static thresholding.}
\label{tab:dwl_dct_combined}
    \resizebox{\linewidth}{!}{%
\begin{tabular}{l|rrrr|rrrr}
\toprule
& \multicolumn{4}{c|}{\textbf{Baseline + DCT}} & \multicolumn{4}{c}{\textbf{DLW + DCT}} \\
\cmidrule(lr){2-5} \cmidrule(lr){6-9}
& \multicolumn{2}{c}{Validation} & \multicolumn{2}{c|}{Inference} & \multicolumn{2}{c}{Validation} & \multicolumn{2}{c}{Inference} \\
$\tau_0$ & $\Delta_{TD}$ & $\Delta_{ED}$ & $\Delta_{TD}$ & $\Delta_{ED}$ & $\Delta_{TD}$ & $\Delta_{ED}$ & $\Delta_{TD}$ & $\Delta_{ED}$ \\
\midrule
$0.90$ & $+230$ & $0$ & $+809$ & $+54$ & $\mathbf{+444}$ & $+1$ & $\mathbf{+1525}$ & $+16$ \\
$0.80$ & $+173$ & $+3$ & $+407$ & $+91$ & $\mathbf{+224}$ & $\mathbf{-3}$ & $\mathbf{+690}$ & $+17$ \\
$0.70$ & $\mathbf{+260}$ & $+1$ & $-257$ & $+116$ & $+153$ & $+3$ & $\mathbf{+508}$ & $+33$ \\
$0.60$ & $\mathbf{+315}$ & $+12$ & $\mathbf{+550}$ & $+97$ & $+89$ & $+6$ & $+364$ & $+22$ \\
$0.50$ & $+230$ & $+2$ & $+226$ & $+58$ & $\mathbf{+251}$ & $+15$ & $\mathbf{+804}$ & $+66$ \\
$0.40$ & $+155$ & $+24$ & $-2$ & $+189$ & $\mathbf{+160}$ & $+12$ & $\mathbf{+513}$ & $+63$ \\
$0.30$ & $+56$ & $+6$ & $+12$ & $-172$ & $+8$ & $\mathbf{-39}$ & $\mathbf{+67}$ & $\mathbf{-169}$ \\
$0.20$ & $+8$ & $\mathbf{-69}$ & $-64$ & $+387$ & $0$ & $0$ & $\mathbf{+10}$ & $+46$ \\
\bottomrule
\end{tabular}%
}
\end{table}

\textbf{Ablations.} Comprehensive ablations in Supplementary Material Section~H show: (1) DLW robust to $\alpha\in[0.5,1.5]$ with $<0.4\%$ mAP variation and the exponential transformation and inverse normalization are both necessary; (2) DLS requires $\lambda_{\text{distant}}>\lambda_{\text{close}}$; optimal $\beta$ reflects dataset geometry; (3) DCT's relative ratio objective achieves best TD-ED trade-off; (4) findings generalize across architectures and datasets.
% %%%%%%%%% Conclusion
\section{Conclusion} \label{sec:conclusion}
We present DepthPrior, a framework that treats depth as prior knowledge for improving object detection on distant and small objects. Our theoretical analysis shows that depth-induced heteroscedasticity creates systematic training bias toward nearby objects. Through three components: Depth-Based Loss Weighting (DLW), Depth-Based Loss Stratification (DLS), and Depth-Aware Confidence Thresholding (DCT), we demonstrate that depth-informed supervision exploits geometric priors without architectural modifications.

DepthPrior's effectiveness stems from three insights. First, treating depth as a prior rather than a fused feature avoids the modality imbalance affecting multi-modal architectures. Second, DLW's batch-level (inter-image) approach enables continuous weighting across diverse scenes, which is effective for datasets with complex depth distributions (SUN RGB-D, MS COCO, VisDrone. DLS's per-image (intra-image) stratification provides explicit control over close/distant trade-offs, which is most effective when datasets exhibit clear depth stratification (KITTI's far-field concentration). Third, DepthPrior addresses multiple failure modes in the detection pipeline. Training-time interventions (DLW and DLS) correct the systematic underrepresentation of distant objects in gradient flow. However, even well-trained detectors still exhibit depth-dependent confidence degradation due to the inherent perceptual difficulty of distant detection. DCT addresses this through learned post-processing that optimizes precision-recall trade-offs across the depth range without retraining. 

\section{Limitations and Future Directions}
DepthPrior relies on depth availability from sensors or monocular estimation. Although self-supervised depth estimation has advanced significantly~\cite{yang2024depth}, estimation errors remain a practical concern. While we provide comprehensive ablations, hyperparameter selection, particularly DLS's weight ratios and split factors, requires dataset-specific tuning. Improvements on small/distant objects create trade-offs with large object performance in some configurations. DCT's validation-set optimization introduces sensitivity to validation set size. Larger validation sets would improve generalization, but the high cost of manual labeling makes them difficult to obtain. Finally, while we use both anchor-based and anchor-free detectors, evaluation on transformer-based architectures remains future work.

Several directions extend this work. Learnable weighting strategies could replace fixed hyperparameters. For example, making $\alpha$ input-dependent for DLW, or learning $(\lambda_{\text{close}}, \lambda_{\text{distant}})$ and $\beta$ for DLS through uncertainty-based weighting~\cite{kendall2018multi}. Semi-supervised learning could leverage DepthPrior to improve pseudo-label quality in self-training~\cite{sbeytibuilding}. Depth-normalized evaluation metrics could provide fairer cross-dataset comparison than fixed area-based metrics. Finally, the paradigm of supervision-level integration may extend to other geometric cues (surface normals, optical flow) or semantic priors.

\section*{Acknowledgments}
This work has been funded by the Pilot Program for Core-Informatics (KiKIT) at the KIT of the Helmholtz Association. It was also supported by the Helmholtz Association Initiative and Networking Fund on the HAICORE@KIT partition. Additionally, this work acknowledges the support of Niklas Lutze in running the YOLOv11 experiments.

\bibliographystyle{IEEEtran}
\bibliography{main}

\newpage

\section{Biography Section}
\begin{IEEEbiography}[{\includegraphics[width=1in,height=1.25in,clip,keepaspectratio,trim={0 13cm 0 0}]{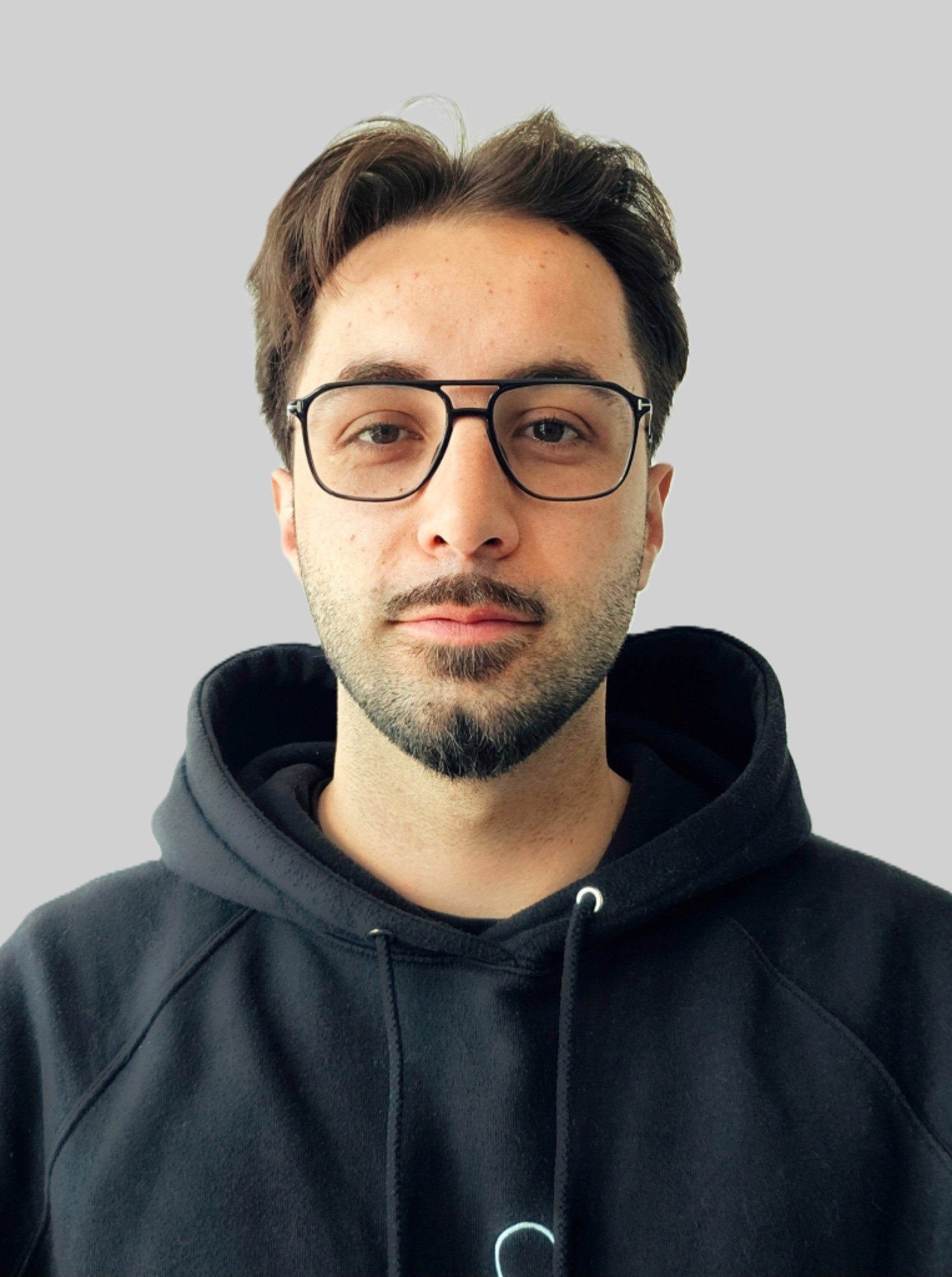}}]{Moussa Kassem Sbeyti}
received his B.Sc. and M.Sc. degrees in Mechatronics Engineering from the Technische Universität Darmstadt, Germany, in 2019 and 2021, respectively. He pursued his Ph.D. in Computer Science at the Technische Universität Berlin (TU Berlin) while working as an AI Research Scientist at the Continental AI Lab (currently part of AUMOVIO SE) in Berlin from 2022 to 2024. He is currently a Postdoctoral AI Researcher with the Methods for Big Data (MBD) research group at the Scientific Computing Center (SCC), Karlsruhe Institute of Technology (KIT), Germany. His research focuses on data-efficient learning for computer vision, particularly uncertainty-based automated labeling methods for 2D object detection in autonomous driving and safety-critical applications.
\end{IEEEbiography}

\begin{IEEEbiography}[{\includegraphics[width=1in,height=1.25in,clip,keepaspectratio,trim={0 1cm 0 0}]{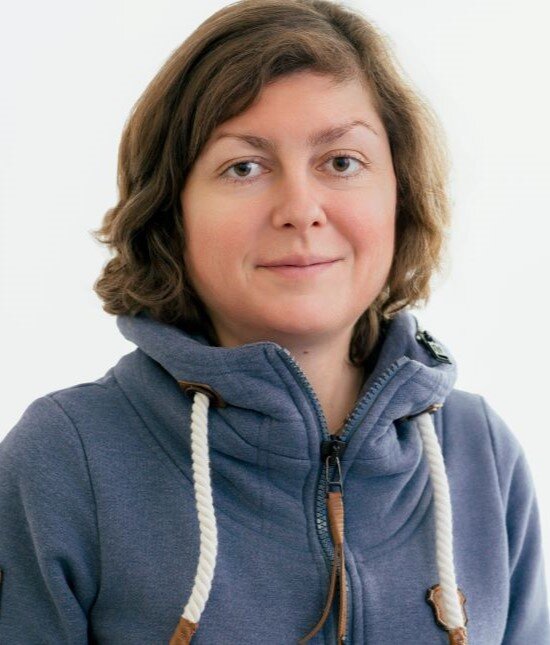}}]{Nadja Klein}
received her Ph.D. degree in Mathematics and Statistics from the Georg-August-Universität Göttingen, Germany, in 2015. She was a Postdoctoral Fellow at the University of Melbourne, Australia, as a Feodor-Lynen Fellow of the Alexander von Humboldt Foundation. She subsequently held the position of Professor of Statistics and Data Science at Humboldt-Universität zu Berlin, Germany. Since August 2024, she has been a Professor at the Scientific Computing Center (SCC), Karlsruhe Institute of Technology (KIT), Germany, where she leads the Methods for Big Data (MBD) research group. She is also Emmy Noether Research Group Leader supported by the Deutsche Forschungsgemeinschaft, and an alumnae of Die Junge Akademie, among  many other prestigious memberships. Nadja’s research and leadership received several famous awards, including the Committee of Presidents of Statistical Societies  Emerging Leader Award in 2025, the Mitchell Prize 2024 of the International Society of Bayesian Analysis, the Wolfgang-Wetzel-Prize 2015 of the German Statistical Society, and numerous awards from the Georg-August-Universität Göttingen and the Universitätsbund Göttingen for her outstanding dissertation. Her research focuses on Bayesian learning as well as statistical and machine learning methods, including uncertainty quantification.
\end{IEEEbiography}
\vfill

\clearpage
\section*{Supplementary Material} \label{sec:appendix}
\subsection{Proofs}
\label{sec:Proof}
\begin{proof}[Proof of Proposition~\ref{prop:depth_heteroscedasticity}]
From Assumption~\ref{assumption:intensity_distance}, the signal quality function is given by:
\begin{equation*}
Q(d) = \frac{\kappa}{d^2}.
\end{equation*}
From Assumption~\ref{assumption:variance_signal}, the conditional loss variance is expressed as:
\begin{equation*}
\text{Var}[\mathcal{L} \mid d] = \frac{\alpha^2}{Q(d)} + \sigma_\varepsilon^2.
\end{equation*}
Substituting the expression for $Q(d)$ gives:
\begin{equation*}
\text{Var}[\mathcal{L} \mid d] = \frac{\alpha^2}{\frac{\kappa}{d^2}} + \sigma_\varepsilon^2 = \frac{\alpha^2 d^2}{\kappa} + \sigma_\varepsilon^2.
\end{equation*}
Defining the depth-dependent variance coefficient $\sigma_0^2 = \alpha^2/\kappa$, we obtain:
\begin{equation*}
\text{Var}[\mathcal{L} \mid d] = \sigma_0^2 d^2 + \sigma_\varepsilon^2.
\end{equation*}
\end{proof}

\begin{proof}[Proof of Corollary~\ref{cor:training_bias}]
Toneva et al.~\cite{toneva2019empirical} demonstrate that training samples partition into ``unforgettable'' (learned early, retained) and ``forgettable'' (repeatedly learned and forgotten) samples, with Spearman correlation of $-0.74$ between forgetting frequency and misclassification margin. Furthermore, Arpit et al.~\cite{arpit2017closer} show that networks achieve maximum validation accuracy before fully fitting high-variance (noise) samples, indicating that consistent patterns are learned before variable ones.

By Proposition~\ref{prop:depth_heteroscedasticity}, for $d_1 < d_2$: $\text{Var}[\mathcal{L} \mid d_1] < \text{Var}[\mathcal{L} \mid d_2]$. Nearby objects thus constitute low-variance ``unforgettable'' samples learned early, while distant objects constitute high-variance ``forgettable'' samples. Under fixed training budget with uniform weighting, the optimizer achieves better convergence on the low-variance nearby-object manifold, resulting in systematic underfitting of distant objects.
\end{proof}

\begin{proof}[Proof of Theorem~\ref{thm:optimal_threshold}]
Consider a detection with confidence score $s$ at depth $d$. The expected costs are $C_{\text{FP}} \cdot P_{\text{FP}}(s, d)$ for accepting it and $C_{\text{FN}} \cdot P_{\text{TP}}(s, d)$ for rejecting it. The optimal decision rule accepts the detection if and only if the expected cost of accepting is no greater than the expected cost of rejecting:
\begin{equation*}
C_{\text{FP}} \cdot P_{\text{FP}}(s, d) \leq C_{\text{FN}} \cdot P_{\text{TP}}(s, d).
\end{equation*}
Rearranging yields the likelihood ratio criterion:
\begin{equation*}
\frac{P_{\text{TP}}(s, d)}{P_{\text{FP}}(s, d)} \geq \frac{C_{\text{FP}}}{C_{\text{FN}}}.
\end{equation*}
Under the assumption that $P_{\text{TP}}(s, d)$ is non-decreasing in $s$ (higher confidence scores correspond to higher true positive probability), the optimal threshold $\tau^*(d)$ is the infimum of confidence scores satisfying this inequality.
\end{proof}

\subsection{Dataset Details and Examples}
\label{sec:dataset_details}
\subsubsection{KITTI Dataset}
The KITTI dataset comprises $7,481$ images ($5,985$ training, $598$ at $10\%$, $1496$ validation $20\%$ random split) with an average resolution of $\approx1240\pm6\times375\pm2$ across $7$ object classes: car, van, truck, pedestrian, person\_sitting, cyclist, and tram, representing automotive driving scenarios. Figure~\ref{fig:kitti_examples} shows a representative KITTI scene with the corresponding depth map and GT annotations.

\begin{figure}[htbp]
    \centering
    \includegraphics[width=\columnwidth]{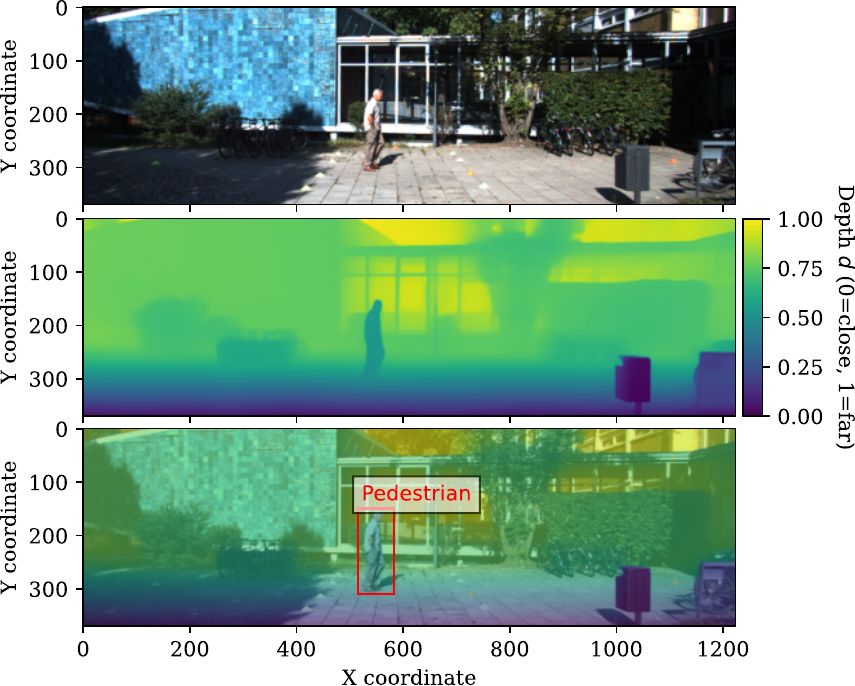}
    \caption{KITTI automotive scene~\cite{geiger2012we}. Top: RGB image. Middle: Predicted depth map using Depth Anything~\cite{yang2024depth}. Bottom: GT annotations overlaid ($40\%$ opacity) on depth visualization.}
    \label{fig:kitti_examples}
\end{figure}

\subsubsection{VisDrone Dataset}
VisDrone contains $10,209$ images ($6,471$ training, $647$ at $10\%$, $548$ validation, $1,610$ test, $1,580$ unlabeled) with an average resolution of $\approx1534\pm280\times1011\pm238$ across $10$ object classes: pedestrian, people, bicycle, car, van, truck, tricycle, awning-tricycle, bus, motor, captured from drone perspectives. Figure~\ref{fig:visdrone_examples} shows an aerial surveillance scene from VisDrone.

\begin{figure}[htbp]
    \centering
    \includegraphics[width=\columnwidth]{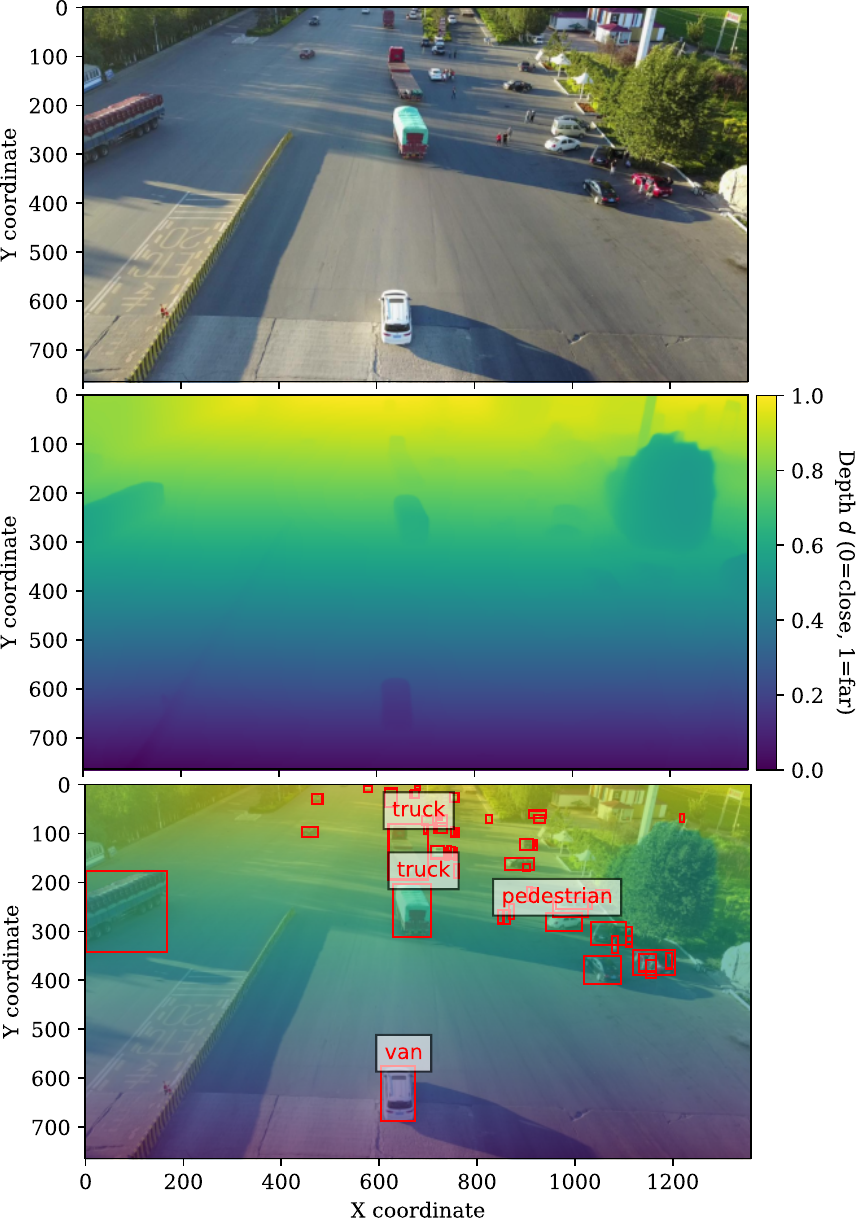}
    \caption{Aerial surveillance example from VisDrone~\cite{du2019visdrone}. Top: RGB image from drone perspective. Middle: Predicted depth map. Bottom: GT annotations overlaid ($40\%$ opacity) on depth map (some class names omitted for clarity).}
    \label{fig:visdrone_examples}
\end{figure}

\subsubsection{SUN RGB-D Dataset}
SUN RGB-D comprises $8,965$ images ($7,154$ training, $715$ at $10\%$, $1,811$ validation) with an average resolution of $\approx650\pm73\times483\pm48$ across $19$ object classes: bed, table, sofa, chair, toilet, desk, dresser, night\_stand, bookshelf, bathtub, box, books, bottle, bag, pillow, monitor, television, lamp, garbage\_bin, representing indoor environments. Figure~\ref{fig:sunrgbd_examples} shows a sample SUN RGB-D indoor scene.

\begin{figure}[htbp]
    \centering
    \includegraphics[width=\columnwidth]{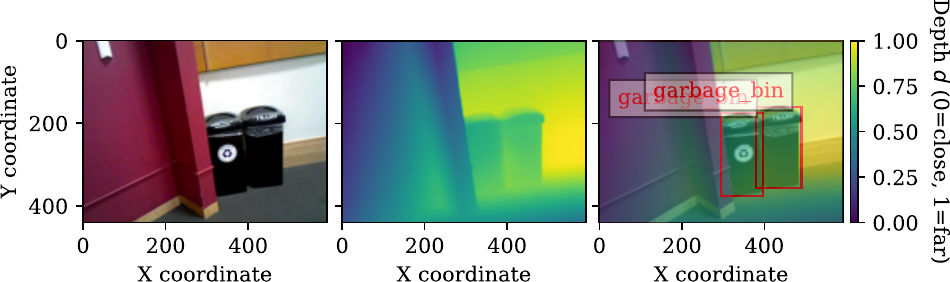}
\caption{SUN RGB-D indoor scene~\cite{song2015sun}. Left: RGB image. Middle: Predicted depth map. Right: GT annotations overlaid ($40\%$ opacity) on depth visualization.}
    \label{fig:sunrgbd_examples}
\end{figure}

\subsubsection{MS COCO Dataset}
MS COCO contains $123,287$ images ($118,287$ training, $11828$ at $10\%$, $5,000$ validation) with an average resolution of $\approx577\pm93\times488\pm97$ across $80$ object classes spanning diverse categories such as person, bicycle, car, motorcycle, airplane, bus, train, truck, boat, traffic light, fire hydrant, stop sign, parking meter, bench, bird, cat, dog, horse, sheep, cow, elephant, bear, zebra, giraffe, backpack, umbrella, handbag, tie, suitcase, frisbee, skis, snowboard, sports ball, kite, baseball bat, baseball glove, skateboard, surfboard, tennis racket, bottle, wine glass, cup, fork, knife, spoon, bowl, banana, apple, sandwich, orange, broccoli, carrot, hot dog, pizza, donut, cake, chair, couch, potted plant, bed, dining table, toilet, tv, laptop, mouse, remote, keyboard, cell phone, microwave, oven, toaster, sink, refrigerator, book, clock, vase, scissors, teddy bear, hair drier, toothbrush, representing general photographic content. Figure~\ref{fig:mscoco_examples} shows a sample MS COCO scene.

\begin{figure}[htbp]
    \centering
    \includegraphics[width=\columnwidth]{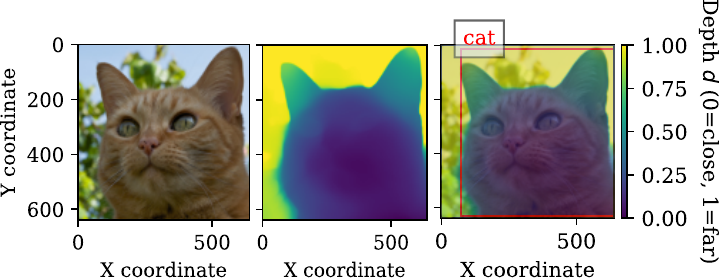}
    \caption{MS COCO photographic scene~\cite{lin2014microsoft}. Left: RGB image. Middle: Predicted depth map. Right: GT annotations overlaid ($40\%$ opacity) on depth visualization.}
    \label{fig:mscoco_examples}
\end{figure}

\subsection{Spatial and Depth Characteristics}
\label{sec:spatial_depth_characteristics}
We analyze spatial and depth patterns across datasets, showing how camera positioning and scene properties create distinct detection challenges that motivate depth-aware supervision. Figure~\ref{fig:dataset_heatmap} shows GT object spatial distributions, revealing systematic spatial bias driven by camera positioning. KITTI exhibits horizontal banding along the horizon (automotive geometry). VisDrone shows center-weighted patterns (aerial perspective). Top-down viewing geometry creates optimal detection conditions in central regions with increasing perspective distortion toward periphery (see Figure~\ref{fig:visdrone_examples}). SUN RGB-D and MS COCO display central concentration (indoor placement and photographic composition, respectively).
\begin{figure}[htbp]
    \centering
    \includegraphics[width=\linewidth]{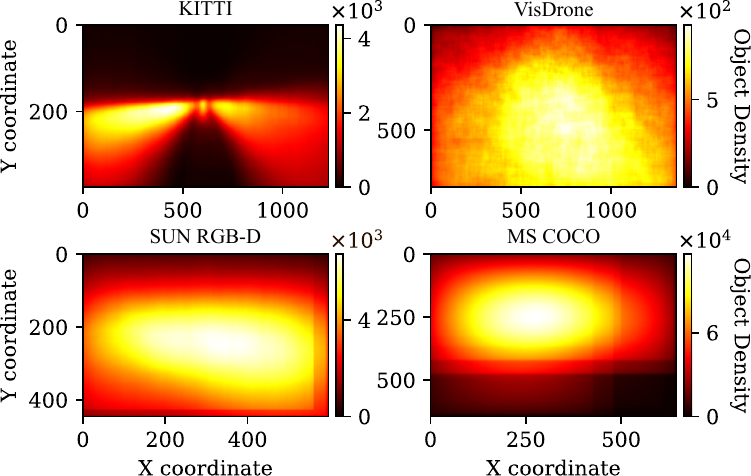}
    \caption{Spatial distribution of objects across datasets derived from GT annotations. (a) KITTI, (b) VisDrone, (c) SUN RGB-D, (d) MS COCO. Heatmaps resized for visualization; spatial resolutions differ across datasets.}
    \label{fig:dataset_heatmap}
\end{figure}

Figure~\ref{fig:dataset_depth_distributions} shows the object depth distributions across all four datasets using predicted depth. While spatial patterns alone appear similar across some contexts (e.g., central clustering in both SUN RGB-D and MS COCO), Figure~\ref{fig:dataset_depth_distributions} reveals different geometric challenges through depth analysis. KITTI displays far-field concentration (depth $0.8-0.9$), reflecting automotive scenarios. VisDrone demonstrates mid-range concentration ($0.3-0.6$), where objects on the ground plane do not vary as much in distance from the camera compared to forward-facing automotive cameras. SUN RGB-D displays steady growth across the depth range, indicating that indoor scenes contain progressively more objects at farther distances due to room geometry. MS COCO exhibits bimodal characteristics with both mid-range presence and far-field concentration, reflecting photographic diversity.
\begin{figure}[htbp]
    \centering
    \includegraphics[width=\linewidth]{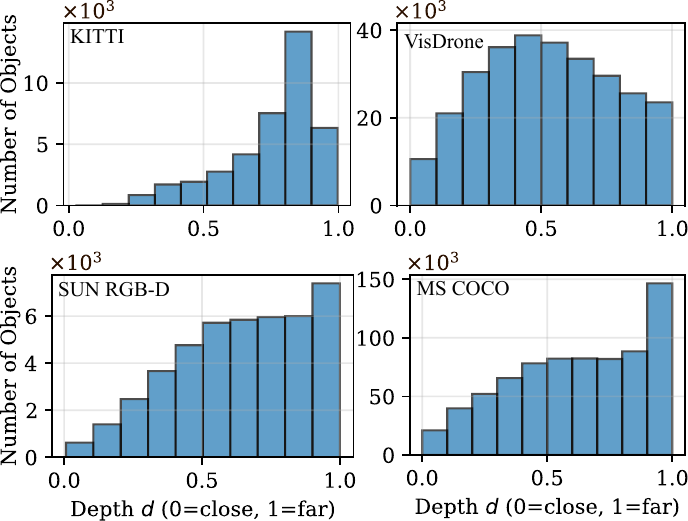}
    \caption{Object depth distribution patterns for all GT objects across datasets using predicted depth from Depth Anything~\cite{yang2024depth}.}
    \label{fig:dataset_depth_distributions}
\end{figure}

The depth distributions reveal information masked by spatial patterns alone. All three datasets, VisDrone, SUN RGB-D, and MS COCO, show central spatial clustering in their heatmaps, suggesting comparable detection scenarios. However, their depth distributions expose different challenges. Far-field objects appear small in the image space regardless of their central spatial position. For example, a $100$ pixel$^2$ object at the image center presents different detection difficulty depending on whether it lies at depth $0.4$ or depth $0.9$. 

Table~\ref{tab:dataset_statistics} quantifies size-depth relationships. KITTI exhibits strong correlation ($r = -0.741$), where depth predicts object scale. SUN RGB-D has the largest median object size ($10,494$ pixels$^2$), but a weak correlation ($r = -0.356$), similar to MS COCO ($r = -0.286$), indicating more varied object scales across the depth range. VisDrone also shows weak correlation ($r = -0.275$) due to uniformly small objects (median $660$ pixels$^2$) from aerial perspective. False detections can have arbitrary sizes. Furthermore, the size of predicted objects is affected by multiple factors beyond distance: object class (cars vs.~pedestrians), occlusion status, and image truncation. In contrast, depth provides a supervision signal independent of detection errors. 

\begin{table}[htbp]
    \centering
    \caption{Dataset statistics: median object area ($\tilde{A}$), median per-image normalized depth ($\tilde{d}$), and Pearson correlation between area and depth ($^{***}$: $p \leq 0.001$).}
    \label{tab:dataset_statistics}
    \begin{tabular}{l|cccc}
        \toprule
        \textbf{Metric} & \textbf{KITTI} & \textbf{VisDrone} & \textbf{SUN RGB-D} & \textbf{MS COCO} \\
        \midrule
        $\tilde{A}$ (pixels$^2$) & $3,573$ & $660$ & $10,494$ & $3,300$ \\
        $\tilde{d}$ & $0.810$ & $0.781$ & $0.760$ & $0.697$ \\
        $r(A,d)$ & $-0.741^{***}$ & $-0.275^{***}$ & $-0.356^{***}$ & $-0.286^{***}$ \\
        \bottomrule
    \end{tabular}
\end{table}

\subsection{Detectors Architecture}
\label{sec:model_analysis}
Our experiments use two different CNN-based detectors. Both YOLOv11~\cite{Jocher_Ultralytics_YOLO_2023} and EfficientDet~\cite{tan2020efficientdet,automl} follow the standard three-part architecture. Backbone for feature extraction, neck for multi-scale feature aggregation (FPN for YOLOv11, BiFPN for EfficientDet), and head for final detection. However, they differ in their anchor systems, loss formulations, and post-processing pipelines.

YOLOv11 employs anchor-free detection, directly predicting class probabilities and bounding box offsets at each feature map location without predefined anchor boxes, avoiding size priors and their constraints. The training loss combines Binary Cross-Entropy for classification with Complete IoU (CIoU) and Distribution Focal Loss (DFL) for box regression. In contrast, EfficientDet relies on anchor-based detection with predefined anchor boxes across multiple scales and aspect ratios. GT assignment uses IoU-based matching, with training loss combining Focal Loss for classification and Smooth L1 Loss for box regression. 

YOLOv11 applies confidence filtering pre-NMS. Therefore, DCT would modulate detection scores before spatial suppression occurs, directly influencing which detections survive NMS. EfficientDet applies confidence filtering post-NMS using soft-NMS, positioning DCT as a final step that refines scores after spatial relationships have been resolved. 

\subsection{Extended Depth-Detection Analysis}
\label{sec:appendix_depth_detection}
This section further investigates the depth-detection relationship, complementing the main analysis in Section~\ref{sec:depth_detection_analysis}.
\subsubsection{Threshold-Independent Error Analysis}
Figure~\ref{fig:error_distribution_appendix} extends Figure~\ref{fig:error_distribution_main} across architectures and splits on MS COCO, confirming MD concentration in distant regions generalizes.
\begin{figure}[htbp]
    \centering
    \includegraphics[width=\linewidth]{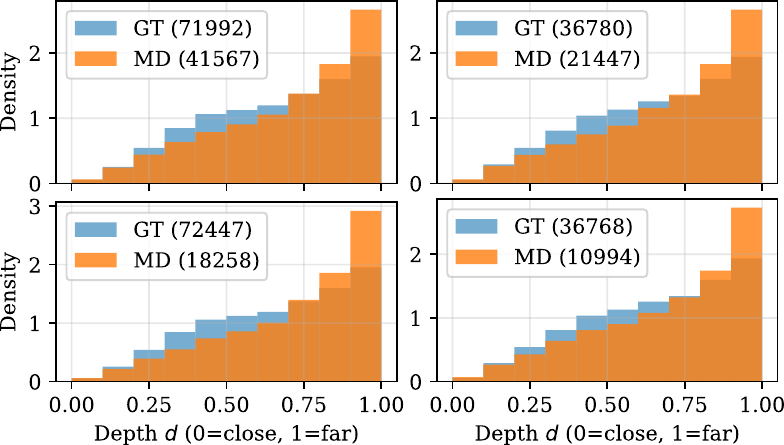}
    \caption{MD patterns on MS COCO. Top: EfficientDet, bottom: YOLOv11. Left: inference set, right: validation set. GT (blue) and MD (orange).}
    \label{fig:error_distribution_appendix}
\end{figure}

\subsubsection{Threshold-Dependent Error Analysis}
Figure~\ref{fig:threshold_analysis_appendix} extends Figure~\ref{fig:threshold_analysis_main} on VisDrone. The threshold-dependent error patterns discussed in Section~\ref{sec:depth_detection_analysis} generalize. YOLOv11 shows favorable MD-ED trade-offs in far depth ranges ($> 0.8$) at $\tau_0 = 0.4$ for both inference and validation, i.e., threshold reduction recovers MD without proportional ED inflation. EfficientDet shows more recovery opportunities in mid-range depths ($0.1$--$0.6$).
\begin{figure}[htbp]
    \centering
    \includegraphics[width=\linewidth]{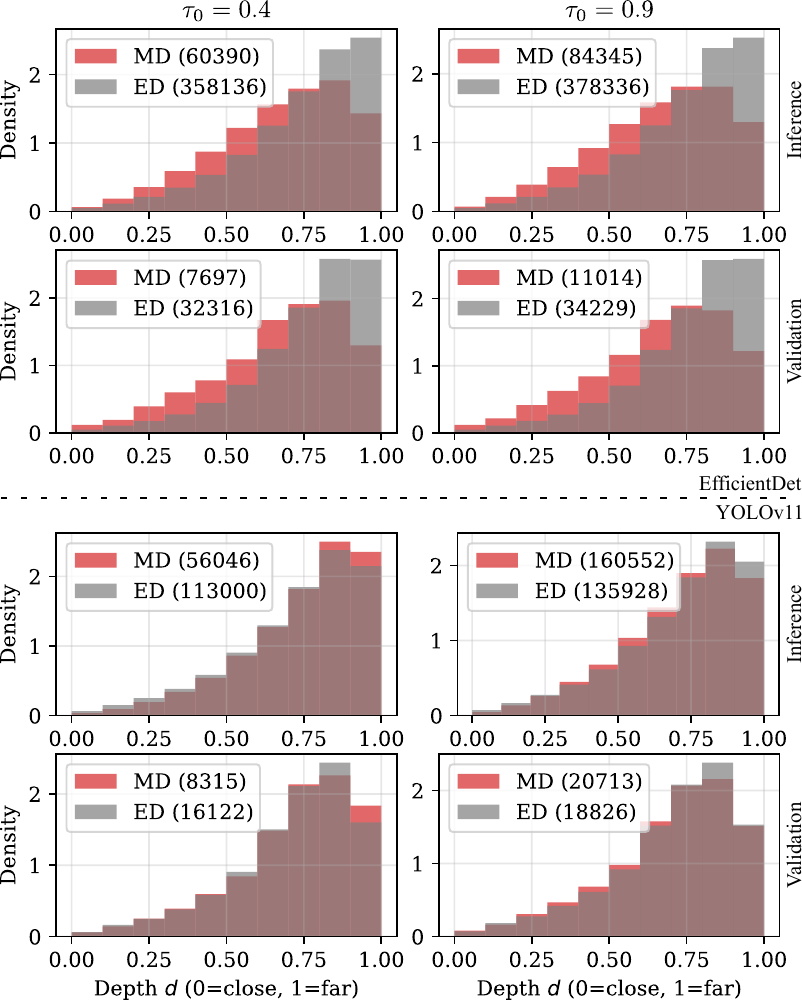}
    \caption{Threshold-dependent error patterns on VisDrone. EfficientDet (rows 1--2) and YOLOv11 (rows 3--4) at $\tau_0 = 0.4$ (left) and $\tau_0 = 0.9$ (right). MD (red) and ED (gray).}
    \label{fig:threshold_analysis_appendix}
\end{figure}

\subsubsection{Confidence-Depth Relationship}
Figure~\ref{fig:confidence_depth_main_appendix} extends Figure~\ref{fig:confidence_depth_main} on KITTI. EfficientDet produces matched detections across broader confidence ranges, with match rates persisting in mid-score regions ($0.4-0.7$) at far depths. This confidence dispersion indicates that EfficientDet's anchor-based training produces less peaked score distributions compared to YOLOv11's anchor-free approach, which concentrates matched detections at higher confidence values ($> 0.7$). Hence, EfficientDet benefits from threshold reduction across wider score ranges, while YOLOv11 requires more conservative adjustment to avoid ED inflation. Both detectors systematically assign lower scores to distant objects even when correctly detected, creating recoverable MD through depth-aware threshold adjustment. Very low confidence regions ($< 0.2$) show minimal match rates across both architectures, establishing a lower bound below which threshold reduction adds primarily ED regardless of depth.

\begin{figure}[htbp]
    \centering
    \includegraphics[width=\linewidth]{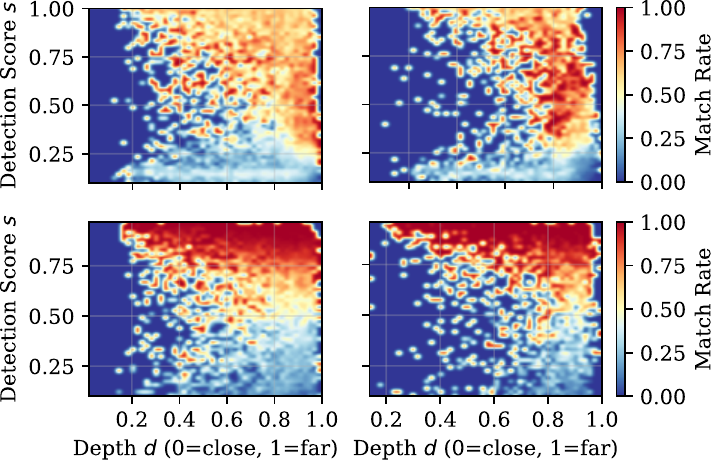}
    \caption{Match rate heatmaps on KITTI. Top: EfficientDet, bottom: YOLOv11. Left: inference set, right: validation set. Color intensity indicates fraction of detections matching GT in each confidence-depth bin.}
    \label{fig:confidence_depth_main_appendix}
\end{figure}

\subsection{DLS Masks Visualization}
\label{sec:dsl_masks}
Figure~\ref{fig:KITTI_dist_masks} visualizes DLS masks at different split factor $\beta$ (25th, 50th, and 75th percentiles of the image's object depth distribution). Lighter regions indicate the distant interval.
\begin{figure}
    \centering
    \includegraphics[width=1\linewidth]{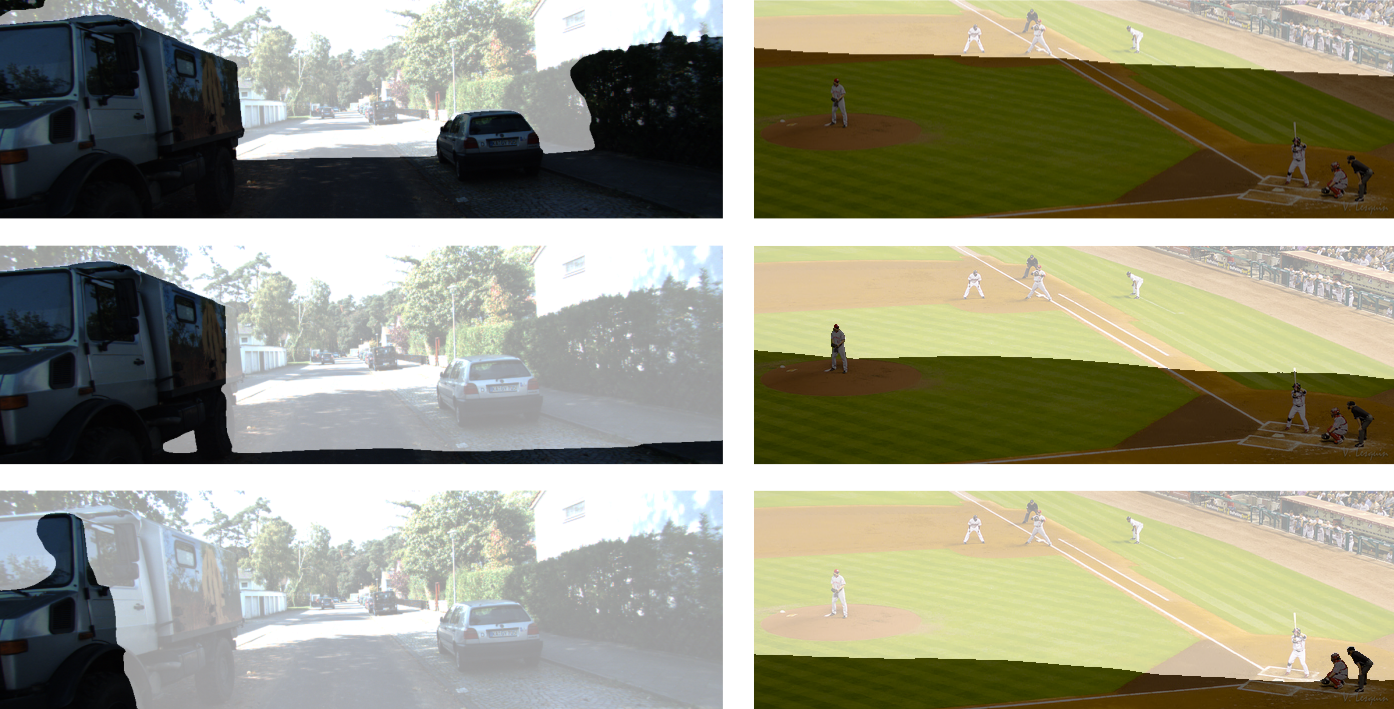}
    \caption{DLS masks on KITTI (left) and MS COCO (right) at $\beta=0.75$, $0.50$, $0.25$ (top to bottom). Lighter regions indicate the distant interval.}
   \label{fig:KITTI_dist_masks}
\end{figure}

\subsection{Extended DCT Analysis}
\label{sec:dct_appendix}
This section provides additional DCT experiments to showcase its robustness across diverse conditions and provide insights into the learned depth-threshold relationships.
\subsubsection{Baseline Performance}
Table~\ref{tab:supp_dct_baseline} presents the baseline mAP of the detectors on $10\%$ of each dataset. The high variability in detection performance across datasets ($3.7\%$ to $47.5\%$) and architectures tests DCT under diverse conditions.
\begin{table}[htbp]
    \centering
    \caption{Baseline detection performance of models used in DCT experiments, trained on $10\%$ of each dataset's training set. Evaluated on the respective validation set.}
    \label{tab:supp_dct_baseline}
    \footnotesize
    \setlength{\tabcolsep}{4pt}
    \begin{tabular}{l|rrrr}
         \toprule
         Detector & KITTI & VisDrone & SUN & COCO \\
         \midrule
         EfficientDet & $30.8$ & $3.7$ & $9.8$ & $9.7$ \\
         YOLOv11 & $47.5$ & $16.0$ & $14.2$ & $33.0$ \\
         \bottomrule
    \end{tabular}
\end{table}

\subsubsection{Optimization and Learned Curves}
Figure~\ref{fig:supp_spline_functions} shows learned threshold functions on KITTI. The learned curves exhibit dataset-specific and architecture-specific characteristics that reflect the empirical patterns established in Section~\ref{sec:depth_detection_analysis}. Figure~\ref{fig:confidence_depth_main} shows that YOLOv11 is rather confident in its predictions with the concentration of most matched detections in high score ranges ($> 0.75$). This behavior is reflected in Figure~\ref{fig:supp_spline_functions} (right).

\begin{figure}[htp]
    \centering
    \includegraphics[width=1\linewidth]{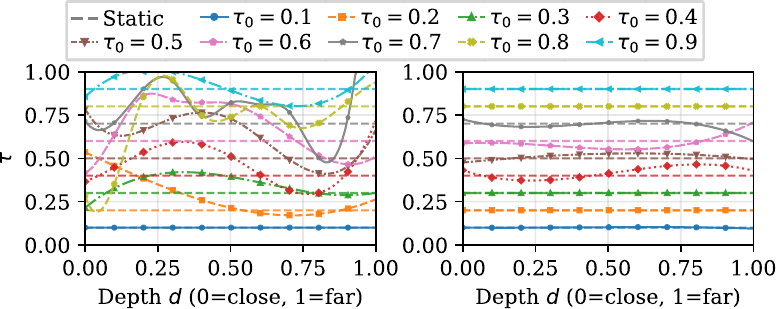}
    \caption{Learned depth-adaptive thresholding curves for KITTI. Left: EfficientDet, Right: YOLOv11.}
    \label{fig:supp_spline_functions}
\end{figure}

Figures~\ref{fig:pareto_kitti} and~\ref{fig:pareto_visdrone} compare Pareto fronts for KITTI and VisDrone with EfficientDet. On KITTI, DCT increases TD in total by $1,427$ while reducing ED by $21$, with higher thresholds enabling more pronounced relaxation due to larger margins before ED inflation. In contrast, improvements on VisDrone distribute evenly across threshold ranges. This validates DCT's adaptive optimization, as it discovers and exploits whatever depth-threshold structure exists for each dataset without requiring manually specified patterns.

\begin{figure}[htp]
    \centering
    \includegraphics[width=0.8\linewidth]{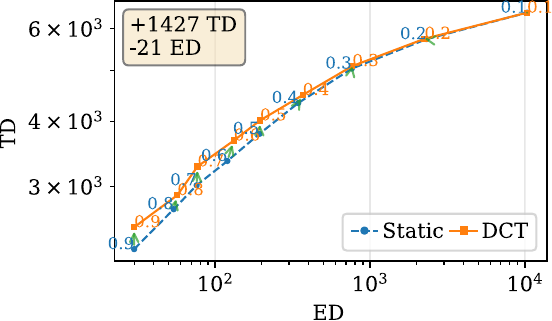}
    \caption{Pareto front comparison between static thresholding (blue) and DCT (orange) on KITTI validation data with EfficientDet. Each point represents a reference threshold $\tau_0 \in \{0.1, 0.2, \dots, 0.9\}$.}
    \label{fig:pareto_kitti}
\end{figure}
\begin{figure}[htp]
        \centering
        \includegraphics[width=0.8\linewidth]{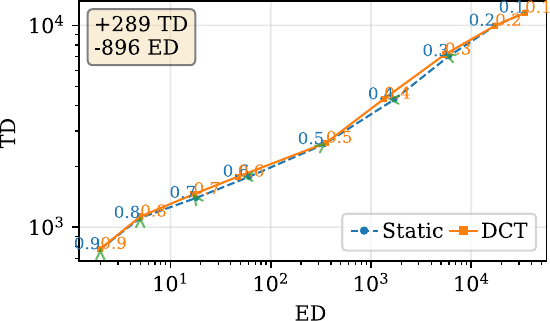}
        \caption{Pareto front comparison on VisDrone validation data with EfficientDet. Static thresholding (blue) and DCT (orange) across $\tau_0 \in \{0.1, 0.2, \dots, 0.9\}$.}
    \label{fig:pareto_visdrone}
    \end{figure}

\subsubsection{Threshold-Specific Generalization}
Tables~\ref{tab:kitti_threshold_sweep} and~\ref{tab:visdrone_threshold_sweep} provide per-threshold breakdowns. Negative deltas on KITTI inference (e.g., $\tau_0\in\{0.2, 0.4, 0.7\}$) arise from the random validation set characteristics, unlike VisDrone's default split. On KITTI, high thresholds ($\tau_0 \geq 0.8$) involve fewer detections and exhibit more consistent generalization, while mid-range thresholds ($0.4-0.7$) include more detections whose depth-confidence characteristics may not be represented in the validation set. 

\begin{table}[hbt]
    \centering
\caption{KITTI threshold sweep for EfficientDet across all reference thresholds. Validation set (top), inference set (bottom). $\tau_0=0.1$ shows zero change as only detections with scores $>0.1$ are retained.}
    \label{tab:kitti_threshold_sweep}
    \resizebox{\columnwidth}{!}{%
\begin{tabular}{l|rrrr|rr}
\toprule
$\tau_0$ & TD & ED & $\text{TD}^*$ & $\text{ED}^*$ & $\Delta_{TD}\uparrow$ & $\Delta_{ED}\downarrow$ \\
\midrule
$0.90$ & $2,281$ & $30$ & $2,511$ & $30$ & $230$ & $0$ \\
$0.80$ & $2,718$ & $54$ & $2,891$ & $57$ & $173$ & $3$ \\
$0.70$ & $3,019$ & $76$ & $3,279$ & $77$ & $260$ & $1$ \\
$0.60$ & $3,359$ & $120$ & $3,674$ & $132$ & $315$ & $12$ \\
$0.50$ & $3,783$ & $193$ & $4,013$ & $195$ & $230$ & $2$ \\
$0.40$ & $4,332$ & $344$ & $4,487$ & $368$ & $155$ & $24$ \\
$0.30$ & $5,039$ & $763$ & $5,095$ & $769$ & $56$ & $6$ \\
$0.20$ & $5,734$ & $2,316$ & $5,742$ & $2,247$ & $8$ & $-69$ \\
$0.10$ & $6,441$ & $10,334$ & $6,441$ & $10,334$ & $0$ & $0$ \\
\midrule
$0.90$ & $15,200$ & $266$ & $16,009$ & $320$ & $809$ & $54$ \\
$0.80$ & $17,459$ & $457$ & $17,866$ & $548$ & $407$ & $91$ \\
$0.70$ & $18,867$ & $666$ & $18,610$ & $782$ & $-257$ & $116$ \\
$0.60$ & $19,922$ & $930$ & $20,472$ & $1,027$ & $550$ & $97$ \\
$0.50$ & $20,838$ & $1,255$ & $21,064$ & $1,313$ & $226$ & $58$ \\
$0.40$ & $21,803$ & $1,740$ & $21,801$ & $1,929$ & $-2$ & $189$ \\
$0.30$ & $22,946$ & $2,668$ & $22,958$ & $2,496$ & $12$ & $-172$ \\
$0.20$ & $24,232$ & $5,587$ & $24,168$ & $5,974$ & $-64$ & $387$ \\
$0.10$ & $25,647$ & $28,920$ & $25,647$ & $28,920$ & $0$ & $0$ \\
\bottomrule
\end{tabular}
}
\end{table}

\begin{table}[hbt]
    \centering
\caption{VisDrone threshold sweep for YOLOv11 across all reference thresholds. Validation set (top), inference set (bottom).}
    \label{tab:visdrone_threshold_sweep}
    \resizebox{\columnwidth}{!}{%
\begin{tabular}{l|rrrr|rr}
\toprule
$\tau_0$ & TD & ED & $\text{TD}^*$ & $\text{ED}^*$ & $\Delta_{TD}\uparrow$ & $\Delta_{ED}\downarrow$ \\
\midrule
$0.90$ & $1,441$ & $2$ & $1,851$ & $3$ & $410$ & $1$ \\
$0.80$ & $5,904$ & $46$ & $6,174$ & $49$ & $270$ & $3$ \\
$0.70$ & $8,342$ & $183$ & $8,665$ & $194$ & $323$ & $11$ \\
$0.60$ & $10,392$ & $448$ & $10,515$ & $477$ & $123$ & $29$ \\
$0.50$ & $12,407$ & $1,000$ & $12,705$ & $1,081$ & $298$ & $81$ \\
$0.40$ & $14,490$ & $2,055$ & $14,825$ & $2,219$ & $335$ & $164$ \\
$0.30$ & $16,727$ & $4,139$ & $16,746$ & $4,152$ & $19$ & $13$ \\
$0.20$ & $19,267$ & $8,523$ & $19,267$ & $8,523$ & $0$ & $0$ \\
$0.10$ & $22,151$ & $18,831$ & $22,151$ & $18,831$ & $0$ & $0$ \\
\midrule
$0.90$ & $11,090$ & $45$ & $14,867$ & $69$ & $3,777$ & $24$ \\
$0.80$ & $50,302$ & $604$ & $52,130$ & $676$ & $1,828$ & $72$ \\
$0.70$ & $72,406$ & $1,998$ & $74,269$ & $2,254$ & $1,863$ & $256$ \\
$0.60$ & $89,370$ & $4,502$ & $90,388$ & $4,714$ & $1,018$ & $212$ \\
$0.50$ & $105,195$ & $9,307$ & $106,908$ & $9,988$ & $1,713$ & $681$ \\
$0.40$ & $120,437$ & $18,132$ & $122,121$ & $19,653$ & $1,684$ & $1521$ \\
$0.30$ & $135,932$ & $34,169$ & $136,014$ & $34,273$ & $82$ & $104$ \\
$0.20$ & $152,555$ & $64,645$ & $152,632$ & $64,865$ & $77$ & $220$ \\
$0.10$ & $171,623$ & $135,992$ & $171,623$ & $135,992$ & $0$ & $0$ \\
\bottomrule
\end{tabular}
}
\end{table}

\subsubsection{Qualitative Examples}
\label{sec:dct_qualitative_appendix}
Figure~\ref{fig:dct_examples_v2} shows DCT examples on SUN RGB-D (indoor truncated/distant furniture) and MS COCO (effective despite weak correlation).
\begin{figure}
        \centering
        \includegraphics[width=\linewidth]{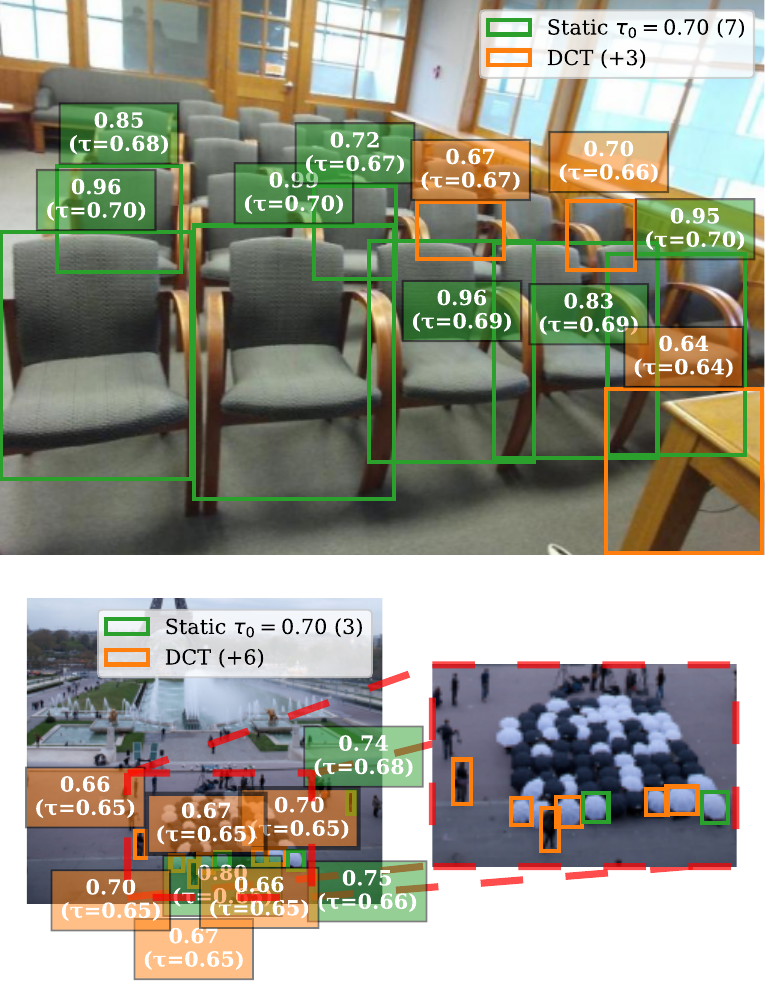}
        \caption{DCT examples on inference data. Top: EfficientDet on SUN RGB-D at $\tau_0=0.70$. Bottom: YOLOv11 on MS COCO at $\tau_0=0.70$. Green boxes: static detections. Orange boxes: DCT recoveries.}
    \label{fig:dct_examples_v2}
    \end{figure}
    
\subsection{Ablation Studies}
\label{sec:ablation_studies}
\subsubsection{DLW Components}
\label{sec:dwl_component_ablation}
We start our ablation studies by ablating each component of DLW in Table~\ref{tab:dwl_component_ablation}: batch-level min-max normalization, inverse transformation ($1 - d$), and exponential weighting ($\exp(d)$).
\begin{table*}[htp]
    \centering
    \caption{DLW component ablation on KITTI (top) and SUN RGB-D (bottom). $D$: raw depth estimator output where far objects receive lower values (Section~\ref{subsec:dwl}). DLW: $w = 1 + \alpha \cdot \exp(d_{\text{norm}})$, $\alpha=1.0$.}
    \label{tab:dwl_component_ablation}
    \resizebox{\textwidth}{!}{%
    \begin{tabular}{l|ccccc|cccc}
         \toprule
         \textbf{Config} & \textbf{mAP} & \textbf{mAP}$_{50}$ & \textbf{mAP}$_S$ & \textbf{mAP}$_M$ & \textbf{mAP}$_L$ & \textbf{mAR} & \textbf{mAR}$_S$ & \textbf{mAR}$_M$ & \textbf{mAR}$_L$ \\
        \midrule
        Baseline & $53.5_{\pm0.3}$ & $81.2_{\pm0.3}$ & $19.7_{\pm0.6}$ & $53.7_{\pm0.1}$ & $72.6_{\pm0.9}$ & $66.8_{\pm0.5}$ & $38.5_{\pm0.5}$ & $66.5_{\pm0.4}$ & $80.7_{\pm0.8}$ \\
        \midrule
        $D$ & $34.2_{\pm2.2}$ & $52.2_{\pm3.3}$ & $7.4_{\pm0.8}$ & $40.8_{\pm1.1}$ & $64.1_{\pm1.3}$ & $57.2_{\pm0.5}$ & $19.0_{\pm1.8}$ & $56.0_{\pm0.8}$ & $78.7_{\pm0.7}$ \\
        $1-d_{\text{norm}}$ & $39.4_{\pm1.1}$ & $59.3_{\pm1.4}$ & $4.8_{\pm1.2}$ & $38.6_{\pm1.8}$ & $68.0_{\pm1.1}$ & $55.0_{\pm0.9}$ & $16.1_{\pm1.2}$ & $53.7_{\pm1.1}$ & $77.3_{\pm0.5}$ \\
        $\exp(1-d_{\text{norm}})$ & $53.3_{\pm0.5}$ & $80.7_{\pm0.3}$ & $19.6_{\pm0.8}$ & $52.8_{\pm0.1}$ & $73.0_{\pm0.6}$ & $66.7_{\pm0.6}$ & $38.1_{\pm1.1}$ & $67.0_{\pm0.6}$ & $80.3_{\pm0.6}$ \\
         $d_{\text{norm}}$ & $54.3_{\pm0.3}$ & $81.8_{\pm0.5}$ & $23.2_{\pm1.7}$ & $54.4_{\pm0.5}$ & $72.9_{\pm0.2}$ & $67.5_{\pm0.3}$ & $41.5_{\pm0.8}$ & $67.1_{\pm0.7}$ & $80.5_{\pm0.6}$\\
        $d_{\text{norm}}^2$ & $54.3_{\pm0.2}$ & $81.7_{\pm0.6}$ & $\mathbf{25.1_{\pm0.3}}$ & $\mathbf{54.6_{\pm0.4}}$ & $72.1_{\pm0.8}$ & $\mathbf{67.6_{\pm0.6}}$ & $\mathbf{42.3_{\pm1.3}}$ & $\mathbf{67.4_{\pm0.3}}$ & $80.4_{\pm0.7}$\\
        \textbf{DLW} & $\mathbf{54.7_{\pm0.2}}$ & $\mathbf{82.6_{\pm0.3}}$ & $23.7_{\pm1.2}$ & $54.1_{\pm0.6}$ & $\mathbf{73.3_{\pm0.3}}$ & $\mathbf{67.6_{\pm0.1}}$ & $41.3_{\pm1.8}$ & $67.3_{\pm0.4}$ & $\mathbf{81.1_{\pm0.3}}$ \\
        \midrule
        \midrule
        Baseline & $17.6_{\pm0.3}$ & $31.3_{\pm0.3}$ & $\mathbf{2.0_{\pm0.2}}$ & $8.9_{\pm0.7}$ & $24.7_{\pm0.3}$ & $32.0_{\pm0.3}$ & $\mathbf{3.8_{\pm0.3}}$ & $20.8_{\pm0.6}$ & $43.7_{\pm1.8}$ \\
        \midrule
        $D$ & $15.8_{\pm0.1}$ & $28.7_{\pm1.1}$ & $0.5_{\pm0.2}$ & $7.0_{\pm0.5}$ & $23.2_{\pm0.3}$ & $32.2_{\pm0.3}$ & $1.0_{\pm0.3}$ & $20.3_{\pm0.5}$ & $43.0_{\pm1.0}$ \\
        $1-d_{\text{norm}}$ & $16.1_{\pm0.2}$ & $28.6_{\pm0.9}$ & $0.6_{\pm0.2}$ & $7.3_{\pm0.2}$ & $23.6_{\pm0.3}$ & $31.2_{\pm1.1}$ & $0.9_{\pm0.1}$ & $19.3_{\pm0.9}$ & $41.0_{\pm1.5}$ \\
        $\exp(1-d_{\text{norm}})$ & $18.8_{\pm0.3}$ & $33.2_{\pm0.2}$ & $1.7_{\pm0.5}$ & $9.4_{\pm0.2}$ & $26.7_{\pm0.4}$ & $\mathbf{35.7_{\pm0.5}}$ & $3.0_{\pm0.6}$ & $23.6_{\pm0.5}$ & $\mathbf{47.6_{\pm1.9}}$ \\
         $d_{\text{norm}}$ & $18.8_{\pm0.3}$ & $33.3_{\pm0.4}$ & $1.5_{\pm0.4}$ & $\mathbf{10.3_{\pm0.1}}$ & $26.4_{\pm0.3}$ & $35.6_{\pm0.5}$ & $3.4_{\pm0.3}$ & $\mathbf{24.3_{\pm0.4}}$ & $46.3_{\pm2.0}$\\
         $d_{\text{norm}}^2$ & $18.2_{\pm0.6}$ & $32.5_{\pm0.7}$ & $1.3_{\pm0.5}$ & $9.8_{\pm0.3}$ & $26.0_{\pm0.5}$ & $35.5_{\pm0.7}$ & $3.4_{\pm0.2}$ & $24.1_{\pm0.9}$ & $45.6_{\pm0.6}$\\
        \textbf{DLW} & $\mathbf{19.1_{\pm0.1}}$ & $\mathbf{33.6_{\pm0.2}}$ & $1.5_{\pm0.2}$ & $9.8_{\pm0.3}$ & $\mathbf{27.0_{\pm0.1}}$ & $35.3_{\pm0.2}$ & $3.5_{\pm0.2}$ & $\mathbf{24.3_{\pm0.4}}$ & $44.6_{\pm0.3}$ \\
         \bottomrule
    \end{tabular}%
    }
\end{table*}

Using the depth estimator's output directly without any normalization, transformation, or inverting ($D$) reduces performance, in accordance with~\cite{cetinkaya2022does}. Performance on KITTI drops to $34.2\%$ mAP ($-19.3\%$), $7.4\%$ mAP$_S$ ($-12.3\%$), and $19.0\%$ mAR$_S$ ($-19.5\%$). Performance on SUN RGB-D also reduces to $15.8\%$ mAP ($-1.8\%$) and $0.5\%$ mAP$_S$ ($-1.5\%$). 

\textbf{Min-Max Normalization.} Adding batch-level min-max normalization ($1-d_{\text{norm}}$) provides some recovery but performance remains below baseline. KITTI achieves $39.4\%$ mAP ($-14.1\%$), SUN RGB-D $16.1\%$ mAP ($-1.5\%$). Small object performance remains low with $4.8\%$ mAP$_S$ ($-14.9\%$) on KITTI and $0.6\%$ mAP$_S$ ($-1.4\%$) on SUN RGB-D.

\textbf{Exponential Weighting.} Adding exponential transformation ($\exp(1-d_{\text{norm}})$) without inverting recovers near-baseline performance. KITTI achieves $53.3\%$ mAP ($-0.2\%$), SUN RGB-D $18.8\%$ mAP ($+1.2\%$). Small object performance approaches baseline (KITTI: $19.6\%$ vs.~$19.7\%$), and recall metrics improve substantially (SUN RGB-D: $35.7\%$ mAR, $+3.7\%$ over baseline). However, this configuration still lacks the distant-object prioritization that inverting provides.

\textbf{Linear and Quadratic Weighting.} Inverse normalization with linear weighting ($d_{\text{norm}}$) achieves competitive overall performance: KITTI $54.3\%$ mAP ($+0.8\%$), SUN RGB-D $18.8\%$ mAP ($+1.2\%$). Small object metrics also improve, with KITTI reaching $23.2\%$ mAP$_S$ ($+3.5\%$) and $41.5\%$ mAR$_S$ ($+3.0\%$). Quadratic weighting ($d_{\text{norm}}^2$) further improves small object precision on KITTI ($25.1\%$ mAP$_S$, $+5.4\%$) but does not improve overall mAP beyond linear weighting. On SUN RGB-D, quadratic weighting slightly underperforms linear ($18.2\%$ vs.~$18.8\%$ mAP), suggesting diminishing returns from stronger nonlinearity on datasets with bounded depth ranges.

\textbf{DLW.} Full DLW with all components achieves the highest overall performance: KITTI $54.7\%$ mAP ($+1.2\%$), SUN RGB-D $19.1\%$ mAP ($+1.5\%$). Compared to linear weighting, DLW improves KITTI mAP by $+0.4\%$ while maintaining comparable small object performance. Compared to quadratic weighting, DLW achieves better overall mAP ($54.7\%$ vs.~$54.3\%$) and large object performance ($73.3\%$ vs.~$72.1\%$ mAP$_L$). 

\textbf{Scaling Factor $\alpha$.} Table~\ref{tab:dwl_alpha_ablation} examines robustness to the $\alpha$ parameter controlling distant objects weight magnitude in $w(d) = 1 + \alpha \cdot \exp(d_{\text{norm}})$. Performance remains stable across the entire tested range. DLW effectiveness does not depend on hyperparameter tuning, enabling rapid deployment. Performance on KITTI varies by only $0.4\%$ mAP ($54.5\%$ to $54.9\%$), with all values outperforming baseline ($+1.0\%$ to $+1.4\%$). VisDrone shows even tighter clustering with $6.4\%$ to $6.5\%$ mAP ($0.1\%$ variation), all improving over baseline ($+0.2\%$ to $+0.3\%$). Small object metrics also exhibit minimal sensitivity with KITTI mAP$_S$ ranging $23.5\%-24.2\%$ (vs.~baseline $19.7\%$) and VisDrone maintaining $1.1\%$ across all $\alpha$ values. Practitioners can adjust based on application priorities: lower values ($\alpha=0.5$) for more conservative reweighting when baseline performance is already strong, higher values ($\alpha=1.5$) for strong far-object emphasis in safety-critical applications.
\begin{table*}[t]
    \centering
    \caption{DLW scaling factor ($\alpha$) ablation on KITTI (top) and VisDrone (bottom).}
    \label{tab:dwl_alpha_ablation}
    \resizebox{\textwidth}{!}{%
    \begin{tabular}{l|ccccc|cccc}
         \toprule
         \textbf{$\alpha$} & \textbf{mAP} & \textbf{mAP}$_{50}$ & \textbf{mAP}$_S$ & \textbf{mAP}$_M$ & \textbf{mAP}$_L$ & \textbf{mAR} & \textbf{mAR}$_S$ & \textbf{mAR}$_M$ & \textbf{mAR}$_L$ \\
        \midrule
        Baseline & $53.5_{\pm0.3}$ & $81.2_{\pm0.3}$ & $19.7_{\pm0.6}$ & $53.7_{\pm0.1}$ & $72.6_{\pm0.9}$ & $66.8_{\pm0.5}$ & $38.5_{\pm0.5}$ & $66.5_{\pm0.4}$ & $80.7_{\pm0.8}$ \\
        \midrule
        $0.5$ & $\mathbf{54.9_{\pm0.4}}$ & $82.4_{\pm0.3}$ & $23.5_{\pm1.6}$ & $\mathbf{54.9_{\pm0.5}}$ & $\mathbf{73.5_{\pm0.7}}$ & $\mathbf{68.0_{\pm0.4}}$ & $41.6_{\pm1.4}$ & $67.9_{\pm0.3}$ & $80.8_{\pm0.7}$ \\
        $1.0$ & $54.7_{\pm0.2}$ & $82.6_{\pm0.3}$ & $23.7_{\pm1.2}$ & $54.1_{\pm0.6}$ & $73.3_{\pm0.3}$ & $67.6_{\pm0.1}$ & $41.3_{\pm1.8}$ & $67.3_{\pm0.4}$ & $\mathbf{81.1_{\pm0.3}}$ \\
        $1.5$ & $54.5_{\pm0.8}$ & $\mathbf{82.7_{\pm0.6}}$ & $\mathbf{24.2_{\pm2.7}}$ & $54.8_{\pm0.5}$ & $72.6_{\pm1.3}$ & $67.7_{\pm0.4}$ & $\mathbf{42.1_{\pm2.3}}$ & $\mathbf{68.0_{\pm0.5}}$ & $80.0_{\pm1.0}$ \\
        \midrule
        \midrule
        Baseline & $6.2_{\pm0.1}$ & $11.2_{\pm0.2}$ & $0.9_{\pm0.0}$ & $9.5_{\pm0.2}$ & $\mathbf{33.4_{\pm0.9}}$ & $9.3_{\pm0.2}$ & $2.7_{\pm0.1}$ & $15.2_{\pm0.4}$ & $40.8_{\pm0.5}$ \\
        \midrule
        $0.5$ & $6.4_{\pm0.1}$ & $11.7_{\pm0.2}$ & $\mathbf{1.1_{\pm0.1}}$ & $9.9_{\pm0.2}$ & $33.3_{\pm1.6}$ & $9.5_{\pm0.1}$ & $2.8_{\pm0.1}$ & $15.4_{\pm0.2}$ & $42.0_{\pm1.3}$ \\
        $1.0$ & $\mathbf{6.5_{\pm0.1}}$ & $\mathbf{12.1_{\pm0.3}}$ & $\mathbf{1.1_{\pm0.1}}$ & $\mathbf{10.0_{\pm0.1}}$ & $33.2_{\pm1.0}$ & $\mathbf{9.7_{\pm0.1}}$ & $2.9_{\pm0.0}$ & $15.7_{\pm0.1}$ & $\mathbf{42.1_{\pm1.3}}$ \\
        $1.5$ & $\mathbf{6.5_{\pm0.1}}$ & $11.9_{\pm0.3}$ & $\mathbf{1.1_{\pm0.0}}$ & $\mathbf{10.0_{\pm0.1}}$ & $32.2_{\pm0.8}$ & $\mathbf{9.7_{\pm0.1}}$ & $\mathbf{3.0_{\pm0.1}}$ & $\mathbf{15.8_{\pm0.1}}$ & $41.9_{\pm1.4}$ \\
         \bottomrule
    \end{tabular}%
    }
\end{table*}
\subsubsection{DLS Components}
\label{sec:dsl_component_ablation}
We continue by ablating each component in DLS: interval weight ratios ($(\lambda_{\text{close}}, \lambda_{\text{distant}})$), and split factor $\beta$.

\textbf{Interval Weight Ratios $(\lambda_{\text{close}}, \lambda_{\text{distant}})$.} Table~\ref{tab:dsl_weight_ablation} examines how the close-to-distant weight ratio $(\lambda_{\text{close}}, \lambda_{\text{distant}})$ affects performance when the split factor $\beta$ is fixed at 50\%.

Moderate distant-object emphasis consistently improves performance across all datasets. DLS on KITTI achieves highest overall mAP at $(1,2)$: $54.5\%$, but $(1,10)$ maximizes small object performance (mAP$_S$: $25.3\%$, $+5.6\%$ over baseline). On SUN RGB-D, DLS maximizes mAP at $(1,2)$: $19.0\%$ mAP ($+1.4\%$). VisDrone demonstrates similar patterns: $(1,5)$ achieves $+0.3\%$ mAP ($6.5\%$) with improvements in small objects ($+0.3\%$ mAP$_S$, $1.2\%$) and medium objects ($+0.4\%$ mAP$_M$, $9.9\%$).

Ratios favoring close objects decrease performance. On KITTI, $(5,1)$ reduces performance to $47.7\%$ mAP ($-5.8\%$), with small object performance reduced to $14.7\%$ mAP$_S$ ($-5.0\%$) and medium objects to $45.6\%$ mAP$_M$ ($-8.1\%$). Even stronger close-object emphasis $(15,1)$ drops performance on KITTI to $40.5\%$ mAP ($-13.0\%$), with small objects at $7.2\%$ mAP$_S$ ($-12.5\%$). VisDrone shows similar directional effects: $(5,1)$ reduces performance to $6.0\%$ mAP ($-0.2\%$), with small objects at $0.7\%$ mAP$_S$ ($-0.2\%$) and medium objects at $9.1\%$ mAP$_M$ ($-0.4\%$). This demonstrates that over-emphasizing close objects creates a training regime where challenging far, small objects are even more neglected. This supports our theoretical motivation in Section~\ref{sec:dd_rel_methods}.

The optimal ratio varies by dataset geometry and imaging perspective. KITTI's far-field concentration tolerates strong distant-object emphasis (($1,10$) still performs well: $54.4\%$ mAP). Meanwhile, SUN RGB-D's more balanced depth distribution requires gentler reweighting as performance reduces beyond ($1,5$): $18.6\%$ $\rightarrow$ $17.8\%$ for ($1,15$). VisDrone's compressed aerial perspective and weak size-depth correlation ($r=-0.275^{***}$) result in smaller improvements but maintain directional consistency, with optimal performance at $(1,5)$ with $\beta=15\%$ ($6.6\%$ mAP). 

\begin{table*}[t]
    \centering
    \caption{DLS interval weight ratio ($\lambda_{\text{close}}, \lambda_{\text{distant}}$) ablation with fixed $\beta=50\%$ on KITTI (top) and SUN RGB-D (bottom).}
    \label{tab:dsl_weight_ablation}
    \resizebox{\textwidth}{!}{%
    \begin{tabular}{cc|ccccc|cccc}
         \toprule
         $\lambda_{\text{close}}$ & $\lambda_{\text{distant}}$ & \textbf{mAP} & \textbf{mAP}$_{50}$ & \textbf{mAP}$_S$ & \textbf{mAP}$_M$ & \textbf{mAP}$_L$ & \textbf{mAR} & \textbf{mAR}$_S$ & \textbf{mAR}$_M$ & \textbf{mAR}$_L$ \\
        \midrule
        \multicolumn{2}{c|}{Baseline} & $53.5_{\pm0.3}$ & $81.2_{\pm0.3}$ & $19.7_{\pm0.6}$ & $53.7_{\pm0.1}$ & $72.6_{\pm0.9}$ & $66.8_{\pm0.5}$ & $38.5_{\pm0.5}$ & $66.5_{\pm0.4}$ & $\mathbf{80.7_{\pm0.8}}$ \\
        \midrule
        1 & 2 & $\mathbf{54.5_{\pm0.3}}$ & $\mathbf{82.7_{\pm0.4}}$ & $24.2_{\pm0.6}$ & $54.7_{\pm0.5}$ & $\mathbf{72.7_{\pm1.5}}$ & $67.6_{\pm0.3}$ & $\mathbf{41.9_{\pm0.2}}$ & $67.4_{\pm0.4}$ & $80.5_{\pm0.5}$ \\
        1 & 5 & $54.4_{\pm0.6}$ & $82.2_{\pm0.7}$ & $24.0_{\pm1.9}$ & $54.5_{\pm0.4}$ & $71.9_{\pm0.9}$ & $67.6_{\pm0.8}$ & $41.8_{\pm2.3}$ & $67.7_{\pm0.7}$ & $80.0_{\pm0.9}$ \\
        1 & 10 & $54.4_{\pm0.5}$ & $82.1_{\pm0.2}$ & $\mathbf{25.3_{\pm0.6}}$ & $\mathbf{54.9_{\pm0.4}}$ & $71.1_{\pm0.9}$ & $\mathbf{67.8_{\pm0.4}}$ & $42.0_{\pm1.0}$ & $\mathbf{68.1_{\pm0.2}}$ & $79.8_{\pm1.4}$ \\
        1 & 15 & $54.0_{\pm0.6}$ & $81.4_{\pm0.4}$ & $24.8_{\pm2.8}$ & $\mathbf{54.9_{\pm0.6}}$ & $70.6_{\pm1.3}$ & $67.3_{\pm0.4}$ & $41.2_{\pm1.2}$ & $\mathbf{68.1_{\pm0.6}}$ & $78.7_{\pm1.6}$ \\
        \midrule
        2 & 1 & $51.9_{\pm0.5}$ & $79.5_{\pm1.0}$ & $17.7_{\pm0.9}$ & $51.3_{\pm0.8}$ & $72.5_{\pm0.5}$ & $65.4_{\pm0.4}$ & $36.8_{\pm0.7}$ & $64.8_{\pm0.5}$ & $80.4_{\pm0.4}$ \\
        5 & 1 & $47.7_{\pm0.3}$ & $75.6_{\pm0.3}$ & $14.7_{\pm1.0}$ & $45.6_{\pm0.2}$ & $70.8_{\pm0.5}$ & $61.5_{\pm0.3}$ & $32.0_{\pm1.2}$ & $60.2_{\pm0.1}$ & $79.1_{\pm1.0}$ \\
        10 & 1 & $42.5_{\pm0.2}$ & $69.4_{\pm0.8}$ & $6.9_{\pm1.3}$ & $38.7_{\pm0.2}$ & $68.8_{\pm0.2}$ & $55.7_{\pm0.2}$ & $23.3_{\pm0.3}$ & $53.7_{\pm0.3}$ & $76.9_{\pm0.6}$ \\
        15 & 1 & $40.5_{\pm0.7}$ & $66.1_{\pm0.6}$ & $7.2_{\pm1.8}$ & $35.9_{\pm0.4}$ & $68.0_{\pm1.2}$ & $53.7_{\pm0.8}$ & $21.5_{\pm1.7}$ & $51.1_{\pm1.1}$ & $76.7_{\pm0.3}$ \\
        \midrule
        \midrule
        \multicolumn{2}{c|}{Baseline} & $17.6_{\pm0.3}$ & $31.3_{\pm0.3}$ & $\mathbf{2.0_{\pm0.2}}$ & $8.9_{\pm0.7}$ & $24.7_{\pm0.3}$ & $32.0_{\pm0.3}$ & $\mathbf{3.8_{\pm0.3}}$ & $20.8_{\pm0.6}$ & $43.7_{\pm1.8}$ \\
        \midrule
        1 & 2 & $\mathbf{19.0_{\pm0.2}}$ & $\mathbf{33.5_{\pm0.2}}$ & $1.6_{\pm0.5}$ & $\mathbf{10.0_{\pm0.8}}$ & $\mathbf{27.0_{\pm0.4}}$ & $\mathbf{35.4_{\pm0.7}}$ & $3.6_{\pm0.9}$ & $23.6_{\pm1.0}$ & $\mathbf{46.1_{\pm0.5}}$ \\
        1 & 5 & $18.6_{\pm0.5}$ & $33.1_{\pm0.9}$ & $1.7_{\pm0.5}$ & $9.9_{\pm0.6}$ & $26.2_{\pm0.9}$ & $34.9_{\pm0.3}$ & $3.7_{\pm0.4}$ & $\mathbf{24.2_{\pm0.2}}$ & $44.4_{\pm0.7}$ \\
        1 & 10 & $18.2_{\pm0.3}$ & $32.1_{\pm0.3}$ & $1.3_{\pm0.5}$ & $9.8_{\pm0.3}$ & $25.0_{\pm0.4}$ & $34.1_{\pm0.7}$ & $3.3_{\pm0.6}$ & $22.9_{\pm0.4}$ & $43.3_{\pm0.8}$ \\
        1 & 15 & $17.8_{\pm0.6}$ & $32.0_{\pm0.9}$ & $1.5_{\pm0.2}$ & $9.2_{\pm0.2}$ & $25.0_{\pm0.5}$ & $33.5_{\pm0.4}$ & $3.3_{\pm0.2}$ & $22.7_{\pm0.6}$ & $42.4_{\pm0.8}$ \\
        \midrule
        2 & 1 & $18.7_{\pm0.3}$ & $32.9_{\pm0.3}$ & $1.3_{\pm0.3}$ & $9.1_{\pm0.2}$ & $26.7_{\pm0.8}$ & $35.0_{\pm0.2}$ & $2.9_{\pm0.2}$ & $23.8_{\pm0.6}$ & $45.9_{\pm1.3}$ \\
        5 & 1 & $18.6_{\pm0.3}$ & $33.1_{\pm0.5}$ & $0.7_{\pm0.1}$ & $9.1_{\pm0.7}$ & $26.2_{\pm0.4}$ & $34.8_{\pm0.8}$ & $2.1_{\pm0.3}$ & $23.7_{\pm0.9}$ & $44.7_{\pm0.8}$ \\
        10 & 1 & $17.0_{\pm0.4}$ & $30.6_{\pm0.3}$ & $0.6_{\pm0.2}$ & $7.9_{\pm0.1}$ & $24.4_{\pm0.2}$ & $33.2_{\pm0.2}$ & $1.6_{\pm0.5}$ & $22.1_{\pm0.4}$ & $42.4_{\pm0.3}$ \\
        15 & 1 & $16.3_{\pm0.1}$ & $29.9_{\pm0.2}$ & $0.7_{\pm0.3}$ & $6.9_{\pm0.2}$ & $23.9_{\pm0.3}$ & $32.0_{\pm0.2}$ & $1.5_{\pm0.2}$ & $20.8_{\pm0.2}$ & $43.4_{\pm1.7}$ \\
         \bottomrule
    \end{tabular}%
    }
\end{table*}

\textbf{Split Factor $\beta$.} Table~\ref{tab:dsl_split_1_5_full} provides a split factor ablation at fixed weight ratio of $(1,5)$ to isolate $\beta$'s contribution.

DLS with $\beta=75\%$, i.e., split at $d\geq75\%$, achieves the highest performance on KITTI ($55.1\%$ mAP, $+1.6\%$ over baseline) and $28.4\%$ mAP$_S$ ($+8.7\%$ over baseline) for small objects. The high split creates a large ``close'' region and small ``distant'' region, aligning with KITTI's far-field concentration (Figures~\ref{fig:dataset_depth_distributions} and~\ref{fig:KITTI_dist_masks}). Lower splits ($50\%$ and $25\%$) reduce improvements on small object progressively ($24.0\%$, $24.5\%$) as more far objects are classified as ``close'', diluting the focused emphasis.

DLS with $\beta=25\%$ achieves the best performance ($18.7\%$ mAP, $+1.1\%$) on SUN RGB-D. The linear depth growth (Figure~\ref{fig:dataset_depth_distributions}) means objects distribute evenly. Higher splits ($50\%$ and $75\%$) achieve comparable mAP but with different trade-offs. DLS with $\beta=50\%$ maximizes mAP$_{50}$ ($33.1\%$) and medium-object recall ($24.2\%$ mAR$_M$). DLS with $\beta=75\%$ maximizes small object recall ($3.9\%$ mAR$_S$), which shows DLS works as intended. 

DLS with $\beta=15\%$ achieves the highest mAP ($6.6\%$, $+0.4\%$) on VisDrone. The mid-range depth concentration ($0.3-0.6$) indicates that most objects lie at similar distances, necessitating a low threshold. Higher splits underperform $\beta=15\%$ because they classify too little mid-range objects as ``far'', diluting emphasis on the challenging objects. $\beta=25\%$ provides a competitive alternative ($10.2\%$ mAP$_M$, $16.1\%$ mAR$_M$), demonstrating robustness around the optimal value.

Optimal split factor reflects depth distribution characteristics. Far-field datasets (KITTI) require high splits, balanced datasets (SUN RGB-D) prefer lower splits, and mid-range aerial datasets (VisDrone) need even lower splits. Practitioners should analyze dataset depth histograms (similar to Figure~\ref{fig:dataset_depth_distributions}) and set splits to isolate the depth interval where detection difficulty concentrates.

\begin{table*}[htp]
    \centering
    \caption{DLS split factor ($\beta$, in \%) ablation with fixed $\lambda_{\text{close}}=1$, $\lambda_{\text{distant}}=5$ on KITTI (top), SUN RGB-D (middle), and VisDrone (bottom).}
    \label{tab:dsl_split_1_5_full} 
    \resizebox{\textwidth}{!}{%
    \begin{tabular}{l|ccccc|cccc}
         \toprule
         \textbf{$\beta$} & \textbf{mAP} & \textbf{mAP}$_{50}$ & \textbf{mAP}$_S$ & \textbf{mAP}$_M$ & \textbf{mAP}$_L$ & \textbf{mAR} & \textbf{mAR}$_S$ & \textbf{mAR}$_M$ & \textbf{mAR}$_L$ \\
        \midrule
        Baseline & $53.5_{\pm0.3}$ & $81.2_{\pm0.3}$ & $19.7_{\pm0.6}$ & $53.7_{\pm0.1}$ & $72.6_{\pm0.9}$ & $66.8_{\pm0.5}$ & $38.5_{\pm0.5}$ & $66.5_{\pm0.4}$ & $80.7_{\pm0.8}$ \\
        \midrule
        75 & $\mathbf{55.1_{\pm0.5}}$ & $\mathbf{82.9_{\pm0.5}}$ & $\mathbf{28.4_{\pm0.5}}$ & $\mathbf{55.7_{\pm0.6}}$ & $71.3_{\pm0.2}$ & $\mathbf{68.3_{\pm0.2}}$ & $\mathbf{45.0_{\pm0.7}}$ & $\mathbf{68.6_{\pm0.1}}$ & $79.3_{\pm0.4}$ \\
        50 & $54.4_{\pm0.6}$ & $82.2_{\pm0.7}$ & $24.0_{\pm1.9}$ & $54.5_{\pm0.4}$ & $71.9_{\pm0.9}$ & $67.6_{\pm0.8}$ & $41.8_{\pm2.3}$ & $67.7_{\pm0.7}$ & $80.0_{\pm0.9}$ \\
        25 & $54.6_{\pm0.4}$ & $81.7_{\pm0.5}$ & $24.5_{\pm2.3}$ & $54.5_{\pm0.2}$ & $\mathbf{73.0_{\pm0.6}}$ & $68.0_{\pm0.4}$ & $41.8_{\pm0.8}$ & $67.8_{\pm0.4}$ & $\mathbf{81.2_{\pm0.5}}$ \\
        \midrule
        \midrule
        Baseline & $17.6_{\pm0.3}$ & $31.3_{\pm0.3}$ & $\mathbf{2.0_{\pm0.2}}$ & $8.9_{\pm0.7}$ & $24.7_{\pm0.3}$ & $32.0_{\pm0.3}$ & $3.8_{\pm0.3}$ & $20.8_{\pm0.6}$ & $43.7_{\pm1.8}$ \\
        \midrule
        75 & $18.4_{\pm0.2}$ & $32.8_{\pm0.5}$ & $1.7_{\pm0.2}$ & $\mathbf{10.0_{\pm0.8}}$ & $25.9_{\pm0.3}$ & $34.8_{\pm0.4}$ & $\mathbf{3.9_{\pm0.3}}$ & $24.1_{\pm0.8}$ & $44.0_{\pm0.2}$ \\
        50 & $18.6_{\pm0.5}$ & $\mathbf{33.1_{\pm0.9}}$ & $1.7_{\pm0.5}$ & $9.9_{\pm0.6}$ & $26.2_{\pm0.9}$ & $34.9_{\pm0.3}$ & $3.7_{\pm0.4}$ & $\mathbf{24.2_{\pm0.2}}$ & $44.4_{\pm0.7}$ \\
        25 & $\mathbf{18.7_{\pm0.4}}$ & $32.8_{\pm0.7}$ & $1.6_{\pm0.3}$ & $9.4_{\pm0.5}$ & $\mathbf{26.6_{\pm0.7}}$ & $\mathbf{35.2_{\pm0.7}}$ & $3.5_{\pm0.3}$ & $\mathbf{24.2_{\pm0.8}}$ & $\mathbf{45.1_{\pm1.1}}$ \\
        \midrule
        \midrule
        Baseline & $6.2_{\pm0.1}$ & $11.2_{\pm0.2}$ & $0.9_{\pm0.0}$ & $9.5_{\pm0.2}$ & $\mathbf{33.4_{\pm0.9}}$ & $9.3_{\pm0.2}$ & $2.7_{\pm0.1}$ & $15.2_{\pm0.4}$ & $40.8_{\pm0.5}$ \\
        \midrule
        85 & $6.2_{\pm0.2}$ & $11.3_{\pm0.3}$ & $1.0_{\pm0.1}$ & $9.3_{\pm0.1}$ & $31.3_{\pm1.1}$ & $9.3_{\pm0.2}$ & $2.8_{\pm0.1}$ & $14.8_{\pm0.3}$ & $41.0_{\pm0.9}$ \\
        75 & $6.2_{\pm0.2}$ & $11.5_{\pm0.4}$ & $1.1_{\pm0.1}$ & $9.3_{\pm0.0}$ & $32.2_{\pm0.5}$ & $9.4_{\pm0.1}$ & $2.9_{\pm0.1}$ & $15.0_{\pm0.2}$ & $40.9_{\pm1.4}$ \\
        50 & $6.5_{\pm0.2}$ & $11.8_{\pm0.1}$ & $\mathbf{1.2_{\pm0.1}}$ & $9.9_{\pm0.1}$ & $32.3_{\pm0.6}$ & $9.7_{\pm0.2}$ & $\mathbf{3.0_{\pm0.1}}$ & $15.6_{\pm0.3}$ & $40.9_{\pm2.0}$ \\
        25 & $6.5_{\pm0.1}$ & $\mathbf{11.9_{\pm0.3}}$ & $1.0_{\pm0.0}$ & $\mathbf{10.2_{\pm0.4}}$ & $32.8_{\pm1.0}$ & $\mathbf{9.8_{\pm0.2}}$ & $2.9_{\pm0.0}$ & $\mathbf{16.1_{\pm0.5}}$ & $\mathbf{41.2_{\pm0.4}}$ \\
        15 & $\mathbf{6.6_{\pm0.2}}$ & $\mathbf{11.9_{\pm0.2}}$ & $1.0_{\pm0.1}$ & $10.0_{\pm0.3}$ & $33.3_{\pm1.2}$ & $9.7_{\pm0.2}$ & $2.8_{\pm0.1}$ & $15.8_{\pm0.5}$ & $40.8_{\pm0.2}$ \\
         \bottomrule
    \end{tabular}%
    }
\end{table*}
\subsubsection{YOLOv11}
The following ablations are on YOLOV11. We first isolate the contributions of DLW components in Table~\ref{tab:yolov11_simple_ablation}. Averaging without inversion underperforms vanilla training. BW ($1-d_\text{norm}$) reduces mAP by $-1.2\%$ ($47.2\%$) and IW ($1-d_\text{norm}$) by $-0.3\%$ ($48.1\%$). This validates that aggregating per-object depth values into scalar batch/image weights eliminates the fine-grained geometric information necessary for effective reweighting. Inversion recovers some performance but remains less than DLW. BW achieves $47.9\%$ ($-0.5\%$ vs.~baseline), IW reaches $48.6\%$ ($+0.2\%$). The improvement over non-inverted versions confirms that prioritizing distant objects is correct, but the continued underperformance demonstrates that simple averaging cannot capture per-object variance patterns. DLW ($1-d_\text{norm}$) without inversion achieves only $47.8\%$ ($-0.6\%$), reducing performance despite maintaining fine-grained per-object information. This proves that granularity alone is insufficient. The weighting must emphasize distant objects. DLW achieves $49.4\%$ ($+1.0\%$), demonstrating that per-object granularity, exponential transformation, and inverse normalization are all necessary. 
\begin{table*}[htp]
    \centering
    \caption{DLW component ablation on KITTI using YOLOv11.}
\label{tab:yolov11_simple_ablation}
    \resizebox{\textwidth}{!}{%
    \begin{tabular}{l|ccccc|cccc}
         \toprule
         \textbf{Method} & \textbf{mAP} & \textbf{mAP}$_{50}$ & \textbf{mAP}$_S$ & \textbf{mAP}$_M$ & \textbf{mAP}$_L$ & \textbf{mAR} & \textbf{mAR}$_S$ & \textbf{mAR}$_M$ & \textbf{mAR}$_L$ \\
        \midrule
        Baseline & $48.4_{\pm0.3}$ & $77.8_{\pm0.4}$ & $31.0_{\pm1.8}$ & $52.4_{\pm0.3}$ & $52.9_{\pm0.8}$ & $58.0_{\pm0.3}$ & $38.1_{\pm1.2}$ & $60.8_{\pm0.6}$ & $63.3_{\pm0.3}$ \\
        \midrule
        BW ($1-d_\text{norm}$) & $47.2_{\pm0.2}$ & $77.2_{\pm0.8}$ & $30.7_{\pm1.7}$ & $50.6_{\pm0.3}$ & $53.1_{\pm0.4}$ & $58.0_{\pm0.4}$ & $38.0_{\pm1.6}$ & $60.5_{\pm0.3}$ & $64.0_{\pm0.9}$ \\
        BW & $47.9_{\pm1.0}$ & $78.1_{\pm0.8}$ & $31.3_{\pm2.5}$ & $51.5_{\pm1.2}$ & $52.8_{\pm0.4}$ & $58.1_{\pm0.9}$ & $39.5_{\pm1.5}$ & $61.0_{\pm0.7}$ & $63.2_{\pm1.6}$ \\
        IW ($1-d_\text{norm}$) & $48.1_{\pm0.3}$ & $77.6_{\pm0.4}$ & $31.2_{\pm0.4}$ & $51.3_{\pm0.6}$ & $53.5_{\pm1.3}$ & $57.9_{\pm0.1}$ & $38.0_{\pm0.5}$ & $60.4_{\pm0.4}$ & $64.3_{\pm1.1}$ \\
        IW & $48.6_{\pm0.2}$ & $77.9_{\pm0.5}$ & $31.5_{\pm1.9}$ & $52.0_{\pm0.5}$ & $52.9_{\pm0.9}$ & $58.2_{\pm0.5}$ & $38.6_{\pm1.5}$ & $60.9_{\pm0.1}$ & $63.6_{\pm1.7}$ \\
        \midrule
        DLW ($1-d_\text{norm}$) & $47.8_{\pm0.5}$ & $77.7_{\pm0.8}$ & $30.7_{\pm0.5}$ & $51.1_{\pm0.5}$ & $53.6_{\pm0.2}$ & $58.3_{\pm0.6}$ & $38.3_{\pm0.7}$ & $60.9_{\pm0.9}$ & $\mathbf{64.4_{\pm0.5}}$ \\
        \textbf{DLW} & $\mathbf{49.4_{\pm0.4}}$ & $\mathbf{78.8_{\pm0.9}}$ & $\mathbf{33.3_{\pm0.2}}$ & $\mathbf{53.3_{\pm0.6}}$ & $\mathbf{53.8_{\pm1.0}}$ & $\mathbf{59.2_{\pm0.1}}$ & $\mathbf{40.0_{\pm0.3}}$ & $\mathbf{62.3_{\pm0.3}}$ & $63.8_{\pm0.7}$ \\
         \bottomrule
    \end{tabular}%
    }
\end{table*}

Cross-architecture consistency with EfficientDet (Table~\ref{tab:dwl_component_ablation}) validates design choices generalize across different detection paradigms. We continue the YOLOv11 ablations by exploring DLS's configuration space in Table~\ref{tab:yolov11_dsl_ablations} to identify optimal settings for YOLOv11. For moderate weighting $(1,2)$, mid-range splits ($\beta=50$) optimize balanced performance, while high splits ($\beta=75$) boost small objects at minimal overall cost ($+0.0\%$ mAP, $+2.2\%$ mAP$_S$). For strong weighting $(1,5)$, split factor becomes critical. High splits ($\beta=75$) create small object focus  ($34.4\%$ mAP$_S$, $+3.4\%$) beneficial for specialized applications but sacrifices overall mAP ($47.4\%$, $-1.0\%$). Mid splits ($\beta=50$) balance improvements with best recall ($59.6\%$ mAR). Low splits ($\beta=25$) maximize overall mAP ($48.9\%$, $+0.5\%$). This demonstrates that weight ratio and split factor interact. Optimal configurations cannot be determined independently but require joint tuning. Unlike EfficientDet which preferred $(1,5)$ weights on KITTI emphasizing distant objects, YOLOv11 achieves competitive performance with various configurations. This suggests anchor-free architectures naturally handle scale variation better.
\begin{table*}[htp]
    \centering
    \caption{DLS configuration ablation on KITTI using YOLOv11.}
\label{tab:yolov11_dsl_ablations}
    \resizebox{\textwidth}{!}{%
\begin{tabular}{ccc|ccccc|cccc}
\toprule
$\lambda_{\text{close}}$ & $\lambda_{\text{distant}}$ & $\beta$ & \textbf{mAP} & \textbf{mAP}$_{50}$ & \textbf{mAP}$_S$ & \textbf{mAP}$_M$ & \textbf{mAP}$_L$ & \textbf{mAR} & \textbf{mAR}$_S$ & \textbf{mAR}$_M$ & \textbf{mAR}$_L$ \\
\midrule
\multicolumn{3}{c|}{\text{Baseline}} & $48.4_{\pm0.3}$ & $77.8_{\pm0.4}$ & $31.0_{\pm1.8}$ & $\mathbf{52.4_{\pm0.3}}$ & $52.9_{\pm0.8}$ & $58.0_{\pm0.3}$ & $38.1_{\pm1.2}$ & $60.8_{\pm0.6}$ & $63.3_{\pm0.3}$ \\
\midrule
1 & 2 & 75 & $48.4_{\pm0.6}$ & $\mathbf{78.6_{\pm0.9}}$ & $\mathbf{33.2_{\pm1.0}}$ & $52.3_{\pm0.9}$ & $52.7_{\pm0.7}$ & $\mathbf{58.3_{\pm0.5}}$ & $\mathbf{39.8_{\pm1.1}}$ & $\mathbf{61.4_{\pm0.7}}$ & $63.3_{\pm1.2}$ \\
1 & 2 & 50 & $\mathbf{48.7_{\pm0.4}}$ & $\mathbf{78.6_{\pm0.4}}$ & $32.4_{\pm1.0}$ & $52.2_{\pm0.4}$ & $53.1_{\pm0.8}$ & $58.2_{\pm0.3}$ & $39.7_{\pm0.8}$ & $60.8_{\pm0.2}$ & $\mathbf{63.5_{\pm0.9}}$ \\
1 & 2 & 25 & $48.2_{\pm0.7}$ & $77.7_{\pm1.8}$ & $31.1_{\pm0.1}$ & $51.3_{\pm0.6}$ & $\mathbf{53.8_{\pm1.3}}$ & $58.2_{\pm0.4}$ & $38.4_{\pm0.6}$ & $60.6_{\pm0.5}$ & $64.2_{\pm0.7}$ \\
\midrule
1 & 5 & 75 & $47.4_{\pm0.8}$ & $77.8_{\pm1.5}$ & $\mathbf{34.4_{\pm0.5}}$ & $51.3_{\pm0.7}$ & $49.9_{\pm1.0}$ & $58.7_{\pm0.4}$ & $\mathbf{41.2_{\pm0.8}}$ & $61.4_{\pm0.9}$ & $63.2_{\pm0.6}$ \\
1 & 5 & 50 & $48.5_{\pm0.4}$ & $\mathbf{78.2_{\pm0.7}}$ & $32.6_{\pm0.4}$ & $52.3_{\pm0.5}$ & $51.9_{\pm0.7}$ & $\mathbf{59.6_{\pm0.7}}$ & $39.4_{\pm0.5}$ & $\mathbf{61.8_{\pm0.6}}$ & $\mathbf{65.7_{\pm1.6}}$ \\
1 & 5 & 25 & $\mathbf{48.9_{\pm0.8}}$ & $78.1_{\pm0.2}$ & $31.9_{\pm1.1}$ & $\mathbf{52.4_{\pm1.0}}$ & $\mathbf{54.1_{\pm0.5}}$ & $58.5_{\pm0.7}$ & $39.0_{\pm1.0}$ & $60.9_{\pm0.7}$ & $64.3_{\pm1.0}$ \\
\bottomrule
\end{tabular}%
    }
\end{table*}
\subsubsection{DCT Optimization Objective}
Table~\ref{tab:dct_ablation} ablates the DCT optimization objective, comparing three versions: (1) basic TD maximization with ED penalty, (2) adding absolute TD/ED ratio term, and (3) our final formulation with relative ratio improvement term. The basic formulation ($-\Delta_{\text{TD}}+\Delta_{\text{ED}}$) achieves the highest TD increase ($+24,356$) but poorest ED control ($+5,282$). Adding the absolute ratio term ($\text{TD}^*/\text{ED}^*$) reduces performance, as the objective shifts toward optimizing the final ratio rather than improving upon the static baseline. Our final formulation with relative ratio improvement ($\text{TD}^*/\text{ED}^* - \text{TD}/\text{ED}$) maintains focus on improvement over the baseline, achieving $+22,042$ TD with only $+2,412$ ED and yielding the highest net benefit ($19,630$).
\begin{table}[t]
    \centering
\caption{DCT optimization objective ablation, aggregated across all datasets (KITTI, VisDrone, SUN RGB-D, MS COCO), architectures (EfficientDet, YOLOv11), and reference thresholds ($\tau_0 \in [0.1, 0.9]$) on held-out inference data. }
\label{tab:dct_ablation}
    \resizebox{\columnwidth}{!}{%
\begin{tabular}{l|cc|c}
\toprule
Objective & $\Delta_{TD}\uparrow$ & $\Delta_{ED}\downarrow$ & $\Delta_{TD}-\Delta_{ED}\uparrow$ \\
\midrule
$-\Delta_{TD}+\Delta_{ED}$ & $\mathbf{24,356}$ & $5,282$ & $19,074$\\
$-\Delta_{TD}+\Delta_{ED}-\frac{\text{TD}^*}{\text{ED}^*}$ & $18,343$ & $5,163$ & $13,180$\\
$-\Delta_{TD}+\Delta_{ED}-(\frac{\text{TD}^*}{\text{ED}^*}-\frac{TD}{ED})$ & $22,042$ & $\mathbf{2,412}$ & $\mathbf{19,630}$\\
\bottomrule
\end{tabular}%
}
\end{table}

\end{document}